\documentclass{article}

\usepackage[preprint, nonatbib]{neurips_2026}
\usepackage[utf8]{inputenc}
\usepackage[T1]{fontenc}
\usepackage{hyperref}
\usepackage{url}
\usepackage{booktabs}
\usepackage{amsfonts}
\usepackage{amsmath}
\usepackage{nicefrac}
\usepackage{microtype}
\usepackage{xcolor}
\usepackage{graphicx}
\usepackage{subcaption}
\usepackage{bm}
\usepackage{authblk}
\usepackage[numbers]{natbib}

\newtheorem{definition}{Definition}
\newtheorem{theorem}{Theorem}
\newtheorem{proposition}{Proposition}
\newtheorem{corollary}{Corollary}[theorem]

\title{Topological Out-of-Domain Generalization in Dynamical Systems Reconstruction}

\author[1,2,$\dagger$,*]{Georg Trede}
\author[1,2,$\dagger$,*]{Charlotte Ricarda Doll}
\author[1,2]{Elias Weber}
\author[1,2,3,$\dagger$]{Daniel Durstewitz}
\affil[1]{Dept. of Theoretical Neuroscience, Central Institute of Mental Health, Mannheim, Germany}
\affil[2]{Faculty of Physics and Astronomy, Heidelberg University, Germany}
\affil[3]{Interdisciplinary Center for Scientific Computing (IWR), Heidelberg University, Germany}
\affil[$\dagger$]{Correspondence to: \{georg.trede, charlotte.doll, daniel.durstewitz\}@zi-mannheim.de}
\affil[*]{Equal contribution}

\begin{document}

\maketitle

\begin{abstract}

Predicting the behavior of dynamical systems (DS) beyond the dynamical and parameter regimes observed in training is a pivotal and essentially unresolved problem in scientific ML. It is central to any good scientific theory, which we expect to be able to make predictions about regimes not covered by currently available data. Recent hierarchical and hyper-network guided approaches for DS reconstruction (DSR) enable training on many DS simultaneously, and revealed that extracted latent features are often related to crucial control parameters of the underlying DS that varied across the training corpus. However, true out-of-domain forecasting abilities of these models, e.g., across tipping points, remain limited, and fine-tuning, or even full model retraining, on time series from the new dynamical regime is usually required. Here, we mathematically analyze the root of these limitations in previous model formulations and identify three core shortcomings rooted in a mismatch between structural assumptions of the reconstruction model and typical properties of physical systems. We propose a combination of remedies for these shortcomings, most importantly \emph{feature splitting}, and furthermore derive a closed-form bound on the reliable extrapolation range. We demonstrate empirically that our techniques allow for accurate zero-shot prediction into new dynamical regimes, outside the observed training regime, as, e.g., encountered across tipping points. 
\end{abstract}

\section{Introduction}
\label{sec:intro}

From turbulent flows and climate systems to ecosystems and neural population activity, much of science is concerned with understanding and predicting the behavior of complex dynamical systems (DS) \cite{brunton_data-driven_2019,durstewitz_reconstructing_2023,gilpin_generative_2024}. A variety of data-driven ML approaches have been proposed toward this goal of dynamical systems reconstruction (DSR) in recent years, ranging from library-based methods like sparse identification of nonlinear dynamical systems (SINDy)~\citep{brunton_discovering_2016, champion_data-driven_2019}, methods based on Neural Ordinary Differential Equations (Neural ODEs)~\citep{chen_neural_2018, karlsson_modelling_2019, alvarez_dynode_2020, ko_homotopy-based_2023}, Koopman operators~\citep{brunton_modern_2021, naiman_koopman_2021, wang_koopman_2022, azencot_forecasting_2020}, Reservoir Computers \citep{pathak_using_2017, platt_systematic_2022}, and physics-informed models~\citep{raissi_physics-informed_2019, cuomo2022scientific}, to methods employing Recurrent Neural Networks (RNNs) ~\citep{hochreiter_lstm_97, vlachas2018data, durstewitz_state_2017,hess_generalized_2023} for approximating the DS' underlying flow map. However, most of these require custom training on any observed DS from scratch and do not easily generalize to new systems, parameter configurations, or dynamical regimes \cite{goring_domain_2024}. In contrast, from a good scientific model built based on only limited data, we would expect that it can predict the behavior of the underlying DS under novel parameter configurations within potentially novel dynamical regimes not previously encountered. For instance, a good biophysical model of the nervous system, even if constructed from biophysical principles derived solely from healthy tissue, would be able to forecast that certain parameter changes lead to epileptic activity \cite{Jirsa2014}. This is a hallmark of scientific theories that is currently unmatched by data-driven ML/AI approaches \cite{wang_generalizing_2023}.

Earlier work on forecasting \textit{non-autonomous} systems has shown some promise in predicting previously unseen dynamics across bifurcations \citep{patel_using_2023, koglmayr2024extrapolating}. However, these approaches typically rely on specific assumptions about the data available, which are empirically often not met, and explicitly couple the dynamical variables to the progression of time, thus do not enable to independently explore ranges of control parameters. 
Ultimately, to achieve true out-of-domain (OOD) generalization (beyond the observed parameter range), reconstructing isolated systems is not sufficient \citep{goring_domain_2024} but DSR models must be able to capture entire \emph{families} of DS across a range of varying control parameters. Indeed, in the last few years, several such frameworks for training across many DS simultaneously \cite{yin2021leads, kirchmeyer2022generalizing, brenner2024learning, vermani2024meta} have surfaced. So far, however, none of them solves the fundamental problem of extrapolating beyond the parameter regime observed in training. 

Here, we mathematically analyze the root of these limitations within the hierarchical DSR framework recently proposed by \cite{brenner2024learning}, but extend the analysis from the originally used piecewise linear RNNs (PLRNNs) to more general discrete-time as well as continuous-time model architectures like Neural ODEs~\citep{chen_neural_2018}. Hierarchical or meta-learning DSR frameworks differentiate between group-level parameters learned across all observed DS, and a small number of latent features which parameterize system-specific DSR models \cite{yin2021leads,kirchmeyer2022generalizing,vermani2024meta,huh_context-informed_2025}. It has been observed empirically that after training on a larger set of DS, the model's latent features map (surprisingly often linearly) onto crucial control parameters that varied within the underlying population of DS \citep{brenner2024learning, huh_context-informed_2025}. While, in theory, this should enable to forecast OOD, in practice fine-tuning or retraining on additional data from the new dynamical regime was still required \citep{brenner2024learning, kirchmeyer2022generalizing}. This indicates that there still must be a mismatch between the model's and the true physical system's structural properties. 

We identify three fundamental structural discrepancies that explain these OOD failures:
(1)~the model Jacobians commonly depend densely on latent features, whereas real systems typically exhibit sparse control-parameter dependence; 
(2)~important geometrical properties of the true system's state space vary in a different way with changes in control parameters than is the case in the DSR model (Fig.~\ref{fig:figure1});
and (3)~the process of time discretization in discrete-time DSR models may induce nonlinear parameter dependencies not present in the true physical DS. 
By fixing the first two issues and establishing bounds on the reliable extrapolation range for the third, we show empirically that we are able to significantly improve \textit{zero-shot} out-of-domain generalization (OODG) across tipping points. We further demonstrate that the general recipe developed here can also be incorporated into other DSR approaches beyond RNNs. Our core contribution is thus a detailed theoretical analysis and solution strategy that significantly improves OODG in DSR without any specific data requirements or strong structural priors in a zero-shot fashion.

\begin{figure}[ht!]
    \centering
    \includegraphics[width=\linewidth]{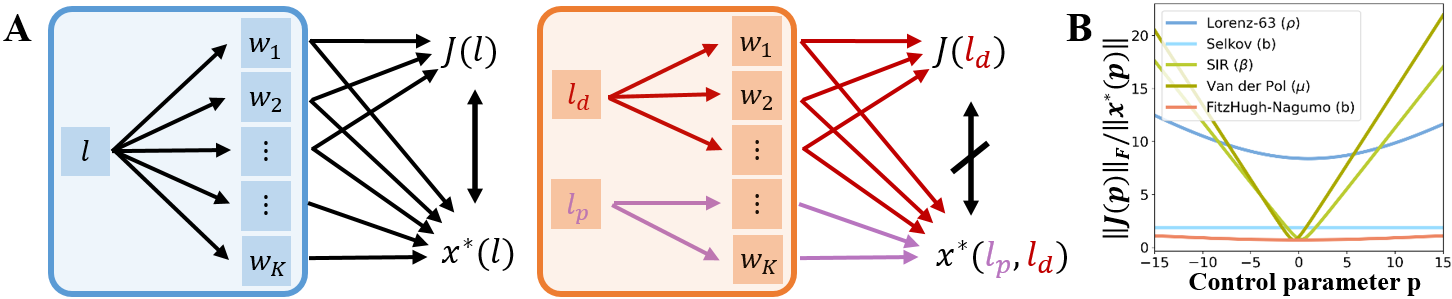}
    \caption{Illustration of hierarchical DSR models. \textbf{(A)} Vanilla hierarchical model (left) uses a single feature vector $\bm{l}$ to obtain weights $\bm{w}_k$ of specific DSR model instantiations, which entangles the scaling of the Jacobian $\bm{J}(\bm{l})$ with the fixed point (FP) position $\bm{x}^*(\bm{l})$ as $\bm{l}$ changes. Conversely, utilizing independent features (right) $\bm{l}_d$ and $\bm{l}_p$ for the local dynamics and FP location enables the decoupled scaling observed in physical systems. \textbf{(B)} Examples of differential Jacobian-vs.-location scaling behavior with control parameter $p$ for different `textbook' examples of DS. $\|\cdot\|_F$ denotes the Frobenius norm; the Jacobian is evaluated at a constant point $\bm{x}_0$.
    }
    \label{fig:figure1}
\end{figure}

\section{Related Work}
\label{sec:relatedwork}

\paragraph{Dynamical Systems Reconstruction (DSR)} 
DSR aims to infer generative surrogate models that capture the invariant (long-term) topological and statistical properties of a true data-generating process \citep{brunton_data-driven_2019, durstewitz_reconstructing_2023}. While library-based approaches like SINDy \citep{brunton_discovering_2016} offer high interpretability, they usually fail on systems which are not covered by their library terms \cite{hess_generalized_2023}. Alternatively, universal approximators of DS such as RNNs \citep{zipser1991recurrent,durstewitz_state_2017,mikhaeil_difficulty_2022} and Neural ODEs \citep{chen_neural_2018, karlsson_modelling_2019, alvarez_dynode_2020, ko_homotopy-based_2023} provide high expressivity. However, training these highly parameterized models on chaotic trajectories is inherently unstable, necessitating advanced control-theoretic strategies \citep{mikhaeil_difficulty_2022,hess_generalized_2023} or special regularization criteria \cite{jiang2023training,platt2023constraining} that ensure long-term statistics and attractor geometries are correctly recovered.

\paragraph{Extrapolating across bifurcations and tipping points} 
In DSR, true OODG means that models should be able to extrapolate to domains \textit{topologically} different from the training regime, such as different basins of attraction of a multistable DS or novel parameter regimes that drive a DS across a tipping point \cite{ashwin2012tipping,goring_domain_2024}. Several studies designed \textit{non-autonomous} (explicitly time-dependent) DSR models based on reservoir computers \cite{kim_teaching_2021, patel_using_2023, koglmayr2024extrapolating} for forecasting dynamical regimes beyond tipping points \cite{ashwin2012tipping}, often enhancing performance through specifically designed training curricula \cite{patel_using_2023} or physics informed approaches \cite{lim_predicting_2020}. However, these approaches remain limited as they often rely on particular assumptions about the structure of the DS or of the data (e.g., by assuming control parameters are explicitly available in training; \cite{kim_teaching_2021}), or by exploiting the behavior of trajectories near critical points or other topological dependencies among dynamical regimes \cite{patel_using_2023}. Similarly, approaches like SINDyCP \citep{nicolaou2023data} achieve cross-bifurcation extrapolation only via a prespecified library of parameter dependencies. This restricts the application of all these methods to scenarios where the (often strong) assumptions are met, leaving the more general problem of predicting across tipping points unresolved.

\paragraph{Learning parameterized DS families} 
While the DSR models described above are all custom-trained on a system at hand, more recent DSR foundation models \cite{hemmer2025true} and frameworks for learning entire families of DS were developed, based on hierarchical and/or hyper-network approaches~\cite{ha2016hypernetworks,von2019continual}. These typically separate shared physical laws from environment-specific parameters via low-dimensional context vectors \citep{yin2021leads, kirchmeyer2022generalizing, panahi2024adaptable} or employ meta-learning across diverse DS families \citep{vermani2024meta, nzoyem2025mixer, brenner2024learning, kong_reservoir_2023}. Importantly in the present context, it has been observed that in hierarchical models \cite{brenner2024learning, huh_context-informed_2025} unsupervised latent features often correlate strongly with the true underlying control parameters, potentially providing a structural basis for knowledge transfer into new dynamical domains.

\section{Background and problem setup}
\label{sec:background}

\subsection{Dynamical Systems}
Following standard DS textbooks \cite{alligood_chaos_1996, guckenheimer_nonlinear_1983, Kuznetsov}, let $T \subseteq \mathbb{R}$ denote continuous time and $X \subseteq \mathbb{R}^M$ the state space of a DS. A DS is defined as a triple $(T, X, \bm{\Phi}_t)$, where $\bm{\Phi}_t : T \times X \to X$ is the flow map which describes the evolution of the DS' state across time $t\in T$ starting from $\bm{x}\in X$. For the common setting of continuous-time DS, the flow map may be thought of as the solution operator of a DS described by ODEs with parameters $\bm{p} \in \mathbb{R}^P$, 
\begin{equation}
    \dot {\bm{x}} = \bm{f}(\bm{x}, \bm{p}).
\label{eq:ode-family}
\end{equation}
While $\bm{p}$ may be continuous, in any empirical setting we only have access to discrete samples, forming a family of DS $\{\mathcal{S}^{(j)}\}_{j=1}^J$ with parameters $\bm{p}^{(j)}$. 

To clarify the idea of topological OODG, we need the concept of topological equivalence \cite{perko_differential_2001,durstewitz2026positiondynamicalsystemsperspective}: 
\begin{definition}[\textbf{Topological equivalence and conjugacy}]\label{def:Def_topoequiv}
Let $(\mathbb{R},A,\Phi)$ and $(\mathbb{R},B,\bm{\Psi})$ be two DS. These are said to be topologically equivalent if there is a homeomorphism $h:A\to B$ such that for each $\bm{x}_0 \in A$ we have $h(\Phi_t(\bm{x}_0))=\Psi_{\tau}(h(\bm{x}_0))$ with $\frac{d\tau(t,\bm{x}_0)}{dt}>0$ everywhere (i.e., if $h$ preserves the direction of flow). They are said to be topologically conjugate, if $\tau(t,\bm{x}_0)=t$ (i.e., if the parameterization by time is the same).
\end{definition}
By \textit{topological} OODG we thus mean that the model is able to generalize to dynamical regimes \textit{not topologically equivalent} to those seen in training. For instance, a bifurcation -- by definition -- implies a topological change in the state space properties \cite{perko_differential_2001}.

\subsection{Target families}
\label{sec:target_family}

For many physical systems the dependence on control parameters -- rates, couplings, forcings, concentrations -- is affine, as $\bm{p}$ enters multiplicatively into the interactions with mechanistic terms in the underlying equations. Canonical examples in this class include the Lorenz-63 \citep{lorenz_deterministic_1963} and Lorenz-96 \citep{lorenz_predictability_1996} models of atmospheric convection, the van-der-Pol \citep{cartwright1960balthazar}, Lotka--Volterra \citep{lotka1920analytical}, and Selkov \citep{selkov1968self} oscillators, the FitzHugh--Nagumo neuron model \citep{fitzhugh1961impulses}, the SIR epidemic model \citep{kermack1927contribution}, or almost any chemical reaction or population ecology system \cite{Murray2002MathBioI,Murray2003MathBioII}. 
Mathematically these systems follow a specific form of Eq.~\ref{eq:ode-family}, given by
\begin{equation}
\dot{\bm{x}} = \bm{f}_0(\bm{x}) + p\, \bm{f}_1(\bm{x})
\qquad
\bm{J}_{\text{cont}}(\bm{x}, p) = \frac{\partial \dot{\bm{x}}}{\partial \bm{x}} =  \bm{J}_0(\bm{x}) + p \, \bm{J}_1(\bm{x}),
\label{eq:true_system}
\end{equation}
where $\bm{x} \in \mathbb{R}^M$, $\bm{f}_0, \bm{f}_1 : \mathbb{R}^M \to \mathbb{R}^M$. For analytical clarity, we focus primarily on the scalar case ($p\in\mathbb{R}$), but we discuss the generalization to multi-dimensional parameters in Appx.~\ref{appx:multidim_parameters}.
Note that DS where $p$ enters nonlinearly (e.g., as $p^2$ or $\sqrt{p}$) into the vector field equations can still be analyzed within our framework after reparameterization ($p\to q:=p^2$ or $p\to q:=\sqrt{p}$, see Appx.~\ref{appx:multidim_parameters}).

\subsection{Hierarchical DSR with affine feature maps}
\label{sec:hierarchical_dsr}
To learn an entire family of DS rather than fitting separate models for each realization, hierarchical and context-conditioned approaches are the most common choice. Building on context-aware dynamics models \citep{brenner2024learning, kirchmeyer2022generalizing, vermani2024meta, yin2021leads}, we represent the parameters of a system-specific model via a shared affine mapping of low-dimensional features $\bm{l}^{(j)} \in \mathbb{R}^{1\times L}$, which may be learned \cite{yin2021leads, brenner2024learning, huh_context-informed_2025} from, or directly provided by, the context, while the mapping itself is defined by fixed group-level parameters $\bm{\theta}_c$, $\bm{\theta}_v$:
\begin{equation}
    \bm{\theta}^{(j)} = \bm{\theta}_c + \bm{l}^{(j)} \bm{\theta}_v = \bm{\theta}_c + \sum_{i=1}^L l^{(j)}_i \bm{\theta}_{v,i},
\label{eq:affine_map}
\end{equation}
Here, $\bm{\theta}_c \in \mathbb{R}^Q$ (centered parameters) and $\bm{\theta}_v \in \mathbb{R}^{L \times Q}$ (variation inducing weights) are shared across all systems, i.e., are independent of $\bm{l}$. $\bm{\theta}_v$ acts as a constant linear map that determines how changes in the latent features $\bm{l}$ translate into shifts in the model parameters $\bm{\theta}^{(j)}$.
Furthermore, it has been demonstrated \citep{kirchmeyer2022generalizing, brenner2024learning, vermani2024meta, huh_context-informed_2025} that such learned features are commonly connected to the true underlying control parameter $p$ via some injective function $\bm{g}$:
\begin{equation}
    \bm{l} = \bm{g}(p),
    \qquad
    \bm{g}:\mathbb{R} \to \mathbb{R}^{1\times L}.
    \label{eq:g}
\end{equation}

While these hierarchical approaches, which separate system-specific from group-level parameters, generally yield strong DSR performance and have a number of nice properties \citep{blanke2023interpretable, brenner2024learning, huh_context-informed_2025, kirchmeyer2022generalizing, yin2021leads}, they meet fundamental structural limitations w.r.t. OODG that are as yet largely unexplored. To analyze these limitations, we consider two representative architectural classes for modeling families of DS. These reflect two complementary perspectives on DSR: learning the flow map directly in discrete time using RNNs, and learning the underlying vector field in continuous time using Neural ODEs. For analytical clarity, we focus primarily on the scalar feature case ($L=1$), but we discuss the generalization to multi-dimensional features in Section~\ref{sec:problem-2}.

\textbf{RNNs (discrete-time flow approximation):}
\begin{equation}
\bm{z}_{t+1} = \bm{F}_{\bm{\theta}}(\bm{z}_t) = \bm{W} \bm{\psi}(\bm{z}_t) + \bm{h} \ ,
\qquad
\bm{J}^{\text{RNN}}_{\bm{\theta}}(\bm{z}_t) = \partial \bm{F}_{\bm{\theta}}(\bm{z}_t)/\partial \bm{z}_t = \bm{W} \,\mathrm{diag}\bigl(\bm{\psi}'(\bm{z}_t)\bigr)
\label{eq:rnn}
\end{equation}
Here, the mapping $\bm{F}_{\bm{\theta}}$ represents a learned approximation of the flow map $\bm{\Phi}$ acting on a latent state $\bm{z}_t \in \mathbb{R}^N$ at the sampled times $t_i=i \Delta t$, with weight parameters $\bm{W} \in \mathbb{R}^{N \times N}$, bias term $\bm{h} \in \mathbb{R}^N$, and an element-wise nonlinearity $\bm{\psi}$. We assume a minimal single-layer architecture for simplicity and as common in the DSR literature \cite{mikhaeil_difficulty_2022,hess_generalized_2023,brenner_almost_2024}, although the theoretical results generalize to more complex networks with hidden layers (Appx.~\ref{appx:feature_scalings}).

\textbf{Neural ODEs (continuous-time vector field approximation):}
\begin{equation}
\dot{\bm{z}}_t = \bm{f}_{\bm{\theta}}(\bm{z}_t) = \bm{W} \bm{\psi} (\bm{z}_t) + \bm{h} \ ,
\qquad
\bm{J}^{\text{NODE}}_{\bm{\theta}}(\bm{z}_t) = \partial \bm{f}_{\bm{\theta}} / \partial \bm{z}_t = \bm{W} \,\mathrm{diag}\bigl(\bm{\psi}'(\bm{z}_t)\bigr)
\label{eq:node}
\end{equation}
In contrast, Neural ODEs parameterize the vector field $\bm{f}_{\bm{\theta}}$ that generates the continuous-time dynamics. For direct comparability, we use the same network structure as for the RNN (Eq.~\ref{eq:rnn}).\footnote{Note that although in Eq. \ref{eq:node} $\bm{z}$ depends continuously on time, for notational consistency we write $t$ as an index.}

Applying the feature mapping (Eq.~\ref{eq:affine_map}) to the RNN and Neural ODE, we obtain the following affine-in-$l$ Jacobian, which governs its local growth attributes, called \textbf{dynamical scaling} in the following:
\begin{equation}
    \bm{J}_{\bm{\theta}} = (\bm{W}_c + l \bm{W}_v) \mathrm{diag}\bigl(\bm{\psi}'(\bm{z}_t)\bigr) \ ,
    \qquad
    \frac{\partial \bm{J}_{\bm{\theta}}}{\partial l} = \bm{W}_v \mathrm{diag}\bigl(\bm{\psi}'(\bm{z}_t) \bigr).
    \label{eq:J_RNN_NODE}
\end{equation}

A key observation now is how the location of the fixed points (FPs) $\bm{z}^*(l)$ depends on the latent feature, which we call \textbf{positional scaling} in the following (referring to a mere \textit{translation} of the vector field that does not alter the local Jacobian). For the discrete-time RNN, the FP must satisfy $\bm{z}^*(l) = \bm{F}_{\bm{\theta}}(\bm{z}^*(l), l)$. Differentiating both sides with respect to $l$ yields:
\begin{equation}
    \frac{d\bm{z}^*}{dl} = \frac{\partial \bm{F}_{\bm{\theta}}}{\partial \bm{z}} \frac{d\bm{z}^*}{dl} + \frac{\partial \bm{F}_{\bm{\theta}}}{\partial l} \implies \frac{d\bm{z}^*}{dl} = \left(\bm{I} - \bm{J}(l)\right)^{-1}\left.\frac{\partial \bm{F}_\theta}{\partial l}\right|_{\bm{z}^*},
    \label{eq:fp_RNN}
\end{equation}
where $\frac{\partial \bm{F}_\theta}{\partial l}$ does not explicitly depend on $l$ (Appx.~\ref{appx:flow_map_dependence}), and we assume $\left(\bm{I} - \bm{J}(l)\right)$ to be invertible, i.e., the Jacobian $\bm{J}(l)$ has no eigenvalue of $+1$, as otherwise we would not have an isolated finite FP anymore (the matrix becomes singular as the eigenvalue approaches 1). For the continuous-time Neural ODE, the equilibrium condition $0 = \bm{f}_{\bm{\theta}}(\bm{z}^*(l), l)$ leads to an analogous implicit derivative (assuming invertibility of $\bm{J}(l)$):
\begin{equation}
    0 = \frac{\partial \bm{f}_{\bm{\theta}}}{\partial \bm{z}} \frac{d\bm{z}^*}{dl} + \frac{\partial \bm{f}_\theta}{\partial l} \implies \frac{d\bm{z}^*}{dl} = - \bm{J}(l)^{-1}\left.\frac{\partial \bm{f}_{\bm{\theta}}}{\partial l}\right|_{\bm{z}^*}
    \label{eq:fp_NODE}
\end{equation}

Thus, the positional scaling depends on the inverse of the Jacobian $\bm{J}(l)$, which is itself a function of $l$. This means that the shift in FP location is mathematically entangled with the change in the system's Jacobian and thus Lyapunov spectrum, which determines the rate of expansion or contraction along different state space directions. This is the basis for the structural bottleneck we formally analyze in Section~\ref{sec:problem-2}.

\section{Structural analysis of affine hierarchical models}
\label{sec:problems}

\subsection{Structural overparameterization and Jacobian sparsity}
\label{sec:problem-1}
In physical systems, a control parameter $p$ typically governs a specific, isolated mechanism, resulting in a sparse derivative of the Jacobian. For instance, the Rayleigh number $\rho$ affects a single entry of the Lorenz-63 Jacobian~\cite{lorenz_deterministic_1963}, the infection rate $\beta$ influences two entries in SIR models~\cite{kermack1927contribution}, and other models such as the Selkov oscillator~\cite{selkov1968self} or the Lorenz-96 system~\cite{lorenz_predictability_1996} have control parameters representing constant inflow or forcing which do not enter the Jacobian at all (i.e., $\bm{J}_1=\bm{0}$). Likewise, in all of the most common bifurcation normal forms the dependence on control parameters is extremely low-dimensional \cite{Kuznetsov}. In contrast, unconstrained affine feature mappings can introduce spurious, dense parameter dependencies that corrupt the topological structure when extrapolated.

\begin{proposition}[\textbf{Perturbations due to dense affine parameter-dependence}]\label{prop:prop_1}
Let the true Jacobian's dependence on $p$ be affine, $\bm{J}_{\text{true}}(p) = \bm{J}_0 + p \bm{J}_{1}$. We assume $\bm{J}_{1} \in \mathcal{S}$, where $\mathcal{S} \subset \mathbb{R}^{N \times N}$ is a $k$-dimensional subspace consisting of matrices that share a fixed, sparse support pattern ($k \ll N^2$), and $\mathcal{S}^\perp$ is its complement. If the DSR model's feature-coupling matrix $\bm{\theta}_v$ (Eq.~\ref{eq:affine_map}) corresponding to the model Jacobian is full-rank and unregularized, it injects $(N^2 - k)$ unconstrained degrees of freedom into the feature-dependent Jacobian $\hat{\bm{J}}_1$. These spurious components lie orthogonal to the true sparse matrix subspace $\mathcal{S}$, allowing an unbounded $\mathcal{O}(|p|)$ spectral divergence outside the training domain.
\end{proposition}

\textit{Proof sketch (Full proof in Appx.~\ref{appx:proof_prop1}.)}$\;$ The model's learned feature-dependent Jacobian term $\hat{\bm{J}}_1$ can be decomposed into the sum of $\hat{\bm{J}}_{1}^{\parallel}\in\mathcal{S}$ and $\hat{\bm{J}}_{1}^{\perp}\in\mathcal{S}^\perp$. Empirical risk minimization (ERM) on finite training data determines $\hat{\bm{J}}_{1}^{\parallel}$, but leaves the $(N^2 - k)$-dimensional orthogonal complement $\hat{\bm{J}}_{1}^{\perp}$ underdetermined. If left unregularized, dense weight matrices $\bm{\theta}_v$ (Eq.~\ref{eq:affine_map}) with $\hat{\bm{J}}_{1}^{\perp}\neq 0$ perturb $\hat{\bm{J}}_1$ out of the $k$-dimensional target space $\mathcal{S}$. By the Bauer--Fike theorem \cite{bauer1960norms}, moving the feature $l \propto p$ outside of the training domain produces linear amplification of such perturbations, which permits spurious eigenvalues to cross the imaginary axis or unit circle and may cause false bifurcations.

\paragraph{Solution: Sparsity and low-rank constraints}
To directly prevent the orthogonal perturbation established in Proposition~\ref{prop:prop_1}, we must restrict the dimensionality of the learned feature-coupling Jacobian $\hat{\bm{J}}_1$. Because true physical systems typically restrict control parameter effects to specific, sparse mechanistic pathways ($k \ll N^2$), dense parameter dependencies in the DSR model formulation allowing $\hat{\bm{J}}_{1}^{\perp}\neq0$ should be actively penalized. Specifically, we propose a low-rank constraint (e.g., with target rank $k'$) jointly with an $\mathcal{L}_1$ regularization on the feature-coupling parameters $\bm{\theta}_v$ (Eq.~\ref{eq:affine_map}) which is added to the training objective:
\begin{align}
    \mathcal{L}_{\mathrm{reg}} = \lambda_\mathrm{reg}\|\bm{\theta}_v\|_1,
    \qquad
    \mathrm{where}\ \bm{\theta}_v = \bm{U}\bm{V}^\top;\ \bm{U} \in \mathbb{R}^{L \times k'},\ \bm{V} \in \mathbb{R}^{Q \times k'}
    \label{eq:sparsity_lr}
\end{align}

Unlike standard ERM, which leaves excess degrees of freedom underdetermined, incorporating this sparsity prior into the training objective actively penalizes the model for populating the physically spurious $(N^2 - k)$-dimensional orthogonal complement $\hat{\bm{J}}_{1}^{\perp}$, and thereby suppresses the unbounded $\mathcal{O}(|p|)$ spectral divergence Proposition~\ref{prop:prop_1} suggests. Crucially, we leave the centered parameters $\bm{\theta}_c$ (Eq.~\ref{eq:affine_map}) unregularized. As $\bm{\theta}_c$ only affects the feature-independent $\hat{\bm{J}}_0$ term of the Jacobian and does not interact with the feature $l$, it does not contribute to the perturbation responsible for the spectral divergence. This targeted regularization thus encourages the training algorithm to align the model's feature-coupling weights almost exclusively with the true sparse matrix subspace.

\paragraph{Empirical validation} Figure~\ref{fig:lorenz_node}A shows a comparison of single-hidden-layer Neural ODEs (Appx.~\ref{appx:used_models_cont}) with and without the proposed sparsity and low-rank constraints, trained on the Lorenz-63 fixed-point regime. To predict dynamics for unseen control parameters, we fit an algebraic function to the latent features learned during training and extrapolate this mapping to the new domain (see Appx.~\ref{appx:feature_extrap} for the full pipeline).

While both models exhibit comparable in-domain performance (measured via Wasserstein-1 distance; Appx.~\ref{appx:evaluation}), their out-of-domain behavior drastically differs. The unconstrained model, exhibiting a dense Jacobian-feature dependence $\partial\bm{J}_\theta/\partial\rho$ (Fig.~\ref{fig:lorenz_node}B, top), suffers from the predicted spectral divergence and fails at the transition to chaos ($\rho\approx24.74$). Conversely, the regularized model maintains a sparser dependence (Fig.~\ref{fig:lorenz_node}B, bottom), successfully suppressing spurious behavior and recreating the ground truth dynamics far beyond the training boundaries. Discrete-time models (Appx.~\ref{appx:disc_model}) exhibit similar behavior (Fig.~\ref{fig:bif_medians_lorenz_shplrnn}, Appx.~\ref{appx:additional_results}).

\begin{figure}[ht!]
    \centering
    \includegraphics[width=\linewidth]{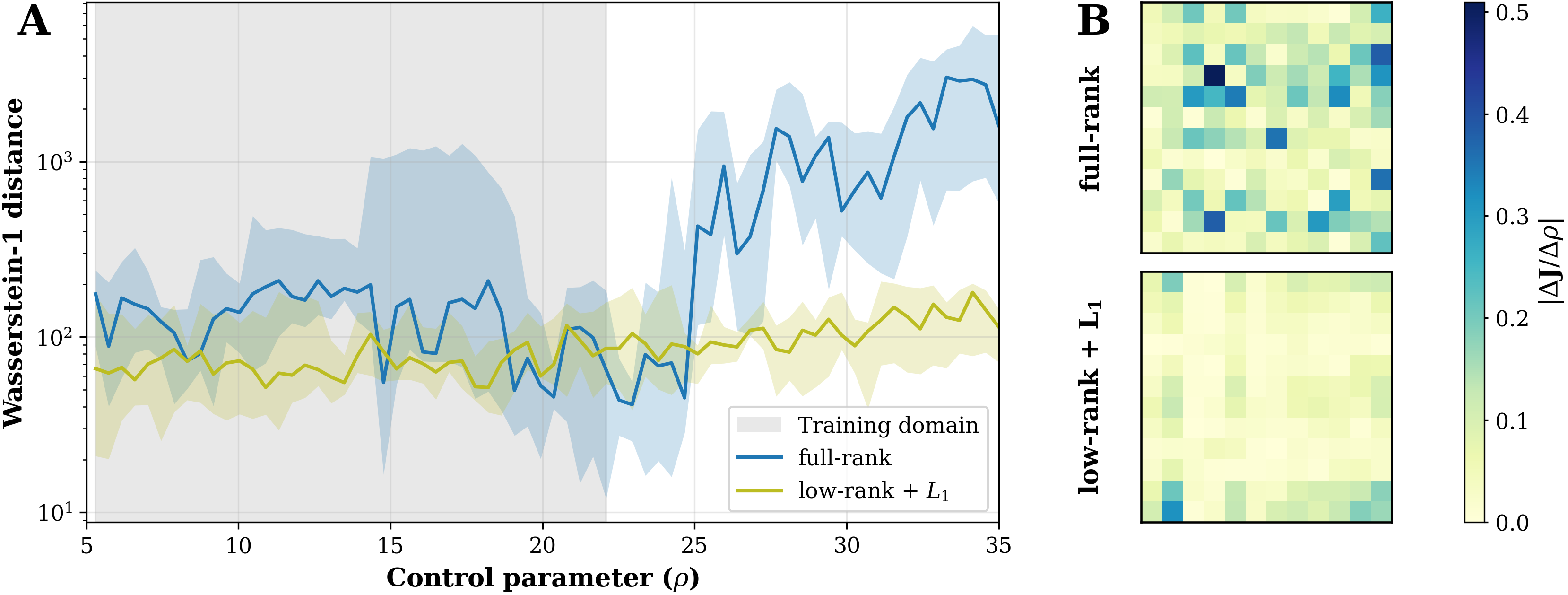}
    \caption{Comparison of low-rank-\&-sparsity-regularized DSR model with standard hierarchization trained on the Lorenz-63 DS with control parameter $\rho$ in the two-fixed-point regime and extrapolated into the chaotic regime. \textbf{(A)} Disagreement in attractor geometry as assessed by Wasserstein-1 distance between true and model-generated bifurcation graphs (Appx.~\ref{appx:evaluation}) for the hierarchical Neural ODE with and w/o low-rank \& sparsity constraints (Eq.~\ref{eq:sparsity_lr}) as a function of $\rho$. Curves are medians across $n=10$ trained models with error bands indicating the inter-quartile range. \textbf{(B)} Change of model Jacobian entries with $\rho$ for the unconstrained (top) and the low-rank-\&-sparsity-regularized model (bottom).}
    \label{fig:lorenz_node}
\end{figure}

\subsection{Asymptotic entanglement of positional and dynamical scaling}
\label{sec:problem-2}

For real physical systems undergoing changes in control parameters, 
the location of attractors (like stable FPs) and the rates of con-/divergence as determined by the Jacobians often vary independently from each other (Fig.~\ref{fig:figure1}B). However, forcing all model parameters to depend on a single feature (Eq.~\ref{eq:affine_map}) algebraically entangles these two properties, as made explicit by the following theorem comparing their asymptotic growth in the limit of large control-parameter values:

\begin{theorem}[\textbf{Asymptotic entanglement in affine feature mappings}]\label{theorem:theorem_1}
Assume that for a true continuous-time DS with affine dependence on a control parameter (Eq.~\ref{eq:true_system}) the asymptotic scaling of a FP and its local Jacobian for $p \to \infty$ is decoupled. Specifically, let their norms follow $\|\bm{x}^*_{\text{true}}(p)\| = \Theta(p^m)$ and $\|\bm{J}_{\text{true}}(p)\| = \Theta(p^n)$ with $m,n\in\mathbb{R}; m \neq n$. Consider a generic DSR model $\bm{F}_{\bm{\theta}}(\bm{z}, l)$ (Eq.~\ref{eq:rnn}) or $\bm{f}_{\bm{\theta}} (\bm{z},l)$ (Eq.~\ref{eq:node}) with parameters strictly affine in $l$, which is coupled to $p$ via $l=g(p)$ (Eq.~\ref{eq:g}), and a bounded or polynomially-bounded activation function $\bm{\psi}$. Then such a model cannot simultaneously satisfy both asymptotic scaling laws, resulting in diverging OOD errors.
\end{theorem}

\textit{Proof sketch (Full proof in Appx.~\ref{appx:proof_thm1}.)}$\;$ Matching the true dynamical scaling strictly determines the features $l(p)$ due to the explicit affine dependence in Eq.~\ref{eq:J_RNN_NODE}. Because the network coefficients $\bm{W}(l)$ and $\bm{h}(l)$ share this same scaling with $l$ (Eq.~\ref{eq:affine_map}), their structural coupling restricts the FP equations. The positional scaling is determined by the partial derivative $\partial \bm{f}_{\bm{\theta}}/\partial l$ or $\partial \bm{F}_{\bm{\theta}}/\partial l$, respectively, (assuming invertibility of $\bm{J}(l)$ or $(\bm{I}-\bm{J}(l))$, respectively; see Eq.~\ref{eq:fp_RNN} and Eq.~\ref{eq:fp_NODE}), which lacks the degrees of freedom to scale independently of the Jacobian. The only remaining channel through which $\bm{z}^*$ could decouple from $\bm{J}_{\bm{\theta}}$ is the derivative of the activation function $\bm{\psi}'(\bm{z}^*(l))$; restricting $\bm{\psi}$ to be bounded or polynomially bounded constrains the FP $\|\bm{z}^*(l)\|$ to scale at best logarithmically with $l$, so the FP is forced onto a constrained asymptotic path in parameter space and cannot grow independently (i.e., $m$ is strictly coupled to $n$).

\begin{corollary}[\textbf{Finite-distance extrapolation divergence}]\label{cor:cor_1_1}
(Proof in Appx.~\ref{appx:proof_cor1_1}.) Assume the dynamical scaling relations are correctly captured, i.e., the learned Jacobian scales linearly with $p$. Because the model's parameterization is strictly affine, the minimal singular value of its Jacobian grows proportionally, $\sigma_{\min}(\bm{J}_{\text{model}}(p)) \geq c|p|$ (for some constant $c > 0$). Consequently, the norm of the inverse Jacobian decays asymptotically as $\mathcal{O}(|p|^{-1})$. By integrating the fixed-point derivative $\partial \bm{z}^*/\partial p$ (Eqs.~\ref{eq:fp_NODE}, \ref{eq:fp_RNN}), we find that if the activation function $\bm{\psi}$ is bounded (e.g., tanh, \texttt{sigmoid}), the model's FP can grow at best logarithmically, $\|\bm{z}^*(p)\|_2 \leq \mathcal{O}(\ln |p|)$. If $\bm{\psi}$ is instead polynomially bounded of degree 1 (e.g., ReLU, GELU), growth eventually halts and $\|\bm{z}^*(p)\|_2 = \Theta(1)$. This architectural bottleneck prevents the independent tuning of positional and dynamical scaling. In either case, because the model's FP path is structurally constrained to a different functional class than that for the true physical system, their derivatives at the training boundaries $[p_{\text{min}}, p_{\text{max}}]$ (i.e., the range of control parameter values experienced during training) cannot perfectly align, even if the model and true positional scaling matched in-domain. This mismatch incurs a non-zero spatial truncation error for any finite extrapolation $\Delta p > 0$, resulting inevitably in topological reconstruction failure.
\end{corollary}

\begin{corollary}[\textbf{Spectral divergence in multi-feature setup}]\label{cor:cor_1_2}
Assume the affine mapping (Eq.~\ref{eq:affine_map}) with a \textit{multi-dimensional} feature vector $\bm{l} \in \mathbb{R}^L$, and a true physical system governed by a scalar control parameter $p$. While an overparameterized feature space ($L > 1$) can empirically decouple positional and dynamical scaling on a training domain, extrapolation of $\bm{l}(p)$ outside the training domain (where $\|\bm{l}(p)\|_2 \to \infty$) introduces perturbations orthogonal to the true physical subspace. The underdetermined $L-1$ degrees of freedom permit an unbounded $\mathcal{O}(|p|)$ divergence of the model's spectral radius, which produces reconstruction failures.
\end{corollary}

\emph{Proof sketch (Full proof in Appx.~\ref{appx:proof_cor1_2}.)}$\;$ To satisfy conflicting algebraic scaling requirements for a FP, $\|\bm{z}^*\|_2 = \Theta(p^m)$, and the Jacobian, $\|\bm{J}\|_2 = \Theta(p^n)$ with $m \neq n$, the feature mapping $\bm{l}(p)$ must utilize at least two independent dimensions. However, because the true parameter dependence is strictly one-dimensional, the feature-coupling matrix block mapping $\bm{l}$ to the model parameters (Eq.~\ref{eq:affine_map}) is physically underdetermined. Because standard ERM strictly optimizes for in-domain loss without structural priors, it leaves these excess dimensions unconstrained, allowing them to retain non-zero weights. When extrapolating outside the training domain, these spurious feature combinations grow linearly. Analogous to Proposition~\ref{prop:prop_1}, this unconstrained growth inserts dense, non-physical matrix blocks into the model's extrapolated Jacobian. By the Bauer--Fike theorem~\cite{bauer1960norms}, this causes spurious eigenvalues to emerge, which drives the spectral radius to diverge and introduces artificial bifurcations.

\paragraph{Solution: Orthogonal feature splitting}
To resolve the asymptotic entanglement established in Theorem~\ref{theorem:theorem_1} without creating the spectral divergence through unconstrained overparameterization (see Corollary~\ref{cor:cor_1_2}), we introduce \textit{feature splitting}. Instead of projecting a single latent feature onto all parameters or naively expanding the feature dimension, we structurally partition the parameter space into two mutually exclusive sets.

We assign an independent dynamical feature $l_{\text{dyn}}\in\mathbb{R}$ strictly to parameters that govern local growth rates (the Jacobians), and a positional feature $l_{\text{pos}}\in\mathbb{R}$ strictly to parameters that shift the vector field without altering its local derivative:
\begin{align}
    P_{\text{dyn}} &= \left\{\bm{\theta}_i: \frac{\partial \bm{J}_{\text{model}}}{\partial \bm{\theta}_i} \neq \bm{0}\right\}, &
    P_{\text{pos}} &= \left\{\bm{\theta}_i: \frac{\partial \bm{J}_{\text{model}}}{\partial \bm{\theta}_i} = \bm{0}\right\}
    \label{eq:criteria_dyn_pos}
\end{align}

The standard feature mapping in hierarchical models (Eq.~\ref{eq:affine_map}) is thereby partitioned into two affine maps for both classes of parameters:
\begin{align}
    \bm{\theta}_i(l_\mathrm{dyn}, l_\mathrm{pos}) = \begin{cases}
        \bm{\theta}_{c,i} + l_\mathrm{dyn} \bm{\theta}_{v,i}, & \text{if } \bm{\theta}_i \in P_{\text{dyn}} \\
        \bm{\theta}_{c,i} + l_\mathrm{pos} \bm{\theta}_{v,i}, & \text{if } \bm{\theta}_i \in P_{\text{pos}}
    \end{cases}
\end{align}

This targeted structural intervention directly resolves these architectural bottlenecks. First, decoupling $l_{\text{pos}}$ from the Jacobian mapping releases the positional scaling from its logarithmic bound (Corollary~\ref{cor:cor_1_1}), enabling it to track the system's true algebraic growth $\Theta(p^m)$. Second, keeping the domains of $l_{\text{dyn}}$ and $l_{\text{pos}}$ mutually exclusive prevents unconstrained degrees of freedom in $\hat{\bm{J}}_1$, thereby avoiding orthogonal perturbation and spectral divergence (Corollary~\ref{cor:cor_1_2}).

By enforcing this architectural prior, the model is freed to learn completely independent nonlinear relationships to the underlying control parameter (e.g., $l_{\text{dyn}} \propto p$ while $l_{\text{pos}} \propto \sqrt{p - p_{crit}}$ for Lorenz-63). This naturally accommodates the heterogeneous scaling with physical properties as required for true OOD extrapolation, functioning as a robust structural guardrail across various DSR architectures (including RNNs and Neural ODEs).

\paragraph{Empirical validation} Trained on multiple control-parameters of the Selkov system (see Appx.~\ref{appx:selkov}) in the fixed-point regime, both discrete-time and continuous-time models (see Appx.~\ref{appx:used_models} for specifics) \emph{with} feature splitting significantly outperform standard models without splitting in the OOD region, where they not only correctly predict the onset of the limit cycle but also the second Hopf bifurcation back to FP dynamics, thus achieving true topological OODG (Fig.~\ref{fig:selkov-shplrnn}B). Furthermore, note that simply increasing the feature dimension of conventional hierarchical models does not have the same effect (Fig.~\ref{fig:selkov-shplrnn}C).

\begin{figure}[ht!]
    \centering
    \includegraphics[width=\linewidth]{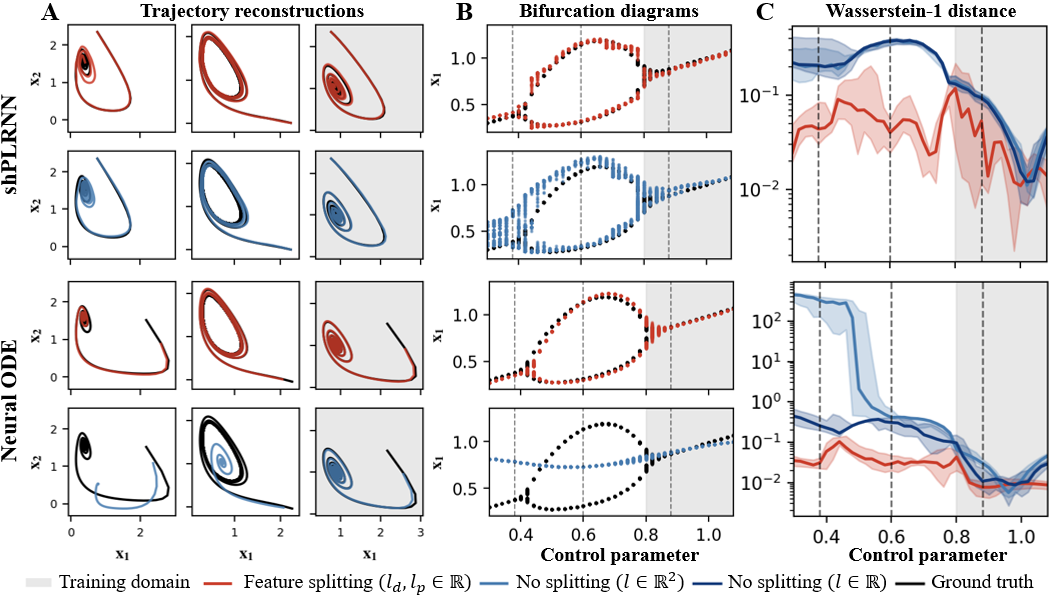}
    \caption{Reconstruction of the Selkov bifurcation structure by models with and without feature splitting. \textbf{(A)} Representative trajectory reconstructions at three distinct control parameter values marked by vertical lines in B and C. \textbf{(B)} Reconstructed bifurcation diagrams (runs with lowest overall Wasserstein-1 distance for each setting). \textbf{(C)} Wasserstein-1 distance between true and reconstructed bifurcation diagrams (Appx.~\ref{appx:evaluation}) for models trained with ($\bm{l}\in\mathbb{R}^1$) and without ($\bm{l}\in\mathbb{R}^1,\mathbb{R}^2$) feature splitting. Solid lines are medians across 10 runs and error bands indicate the interquartile range. The shaded area indicates the training domain.}
    \label{fig:selkov-shplrnn}
\end{figure}

\subsection{Nonlinear scaling constraints in continuous-to-discrete time mapping}
\label{sec:nonlin_scaling}

Mapping a continuous-time parameterized system $\dot{\bm{x}} = \bm{f}(\bm{x}, p)$ to a discrete-time flow map (as is required for RNNs and other discrete flow approximations, but not for Neural ODEs) transforms a fundamentally affine parameter dependence into a nonlinear one. Let the continuous-time Jacobian be $\bm{J}_{\text{cont}}(\bm{x}, p) = \bm{J}_0(\bm{x}) + p\bm{J}_1(\bm{x})$, where $p$ is the scalar control parameter. The Jacobian of the true discrete-time flow map $\boldsymbol{\Phi}_{\Delta t}$ can be expanded for small $\Delta t$ (details in Appx.~\ref{appx:proof_thm2}):
\begin{equation}
    \bm{J}_{\text{disc}}(\bm{x}, p) = \bm{I} + (\bm{J}_0 + p\bm{J}_1)\Delta t + \frac{1}{2}\left( (\bm{J}_0 + p\bm{J}_1)^2 + \dot{\bm{J}}_0 + p\dot{\bm{J}}_1 \right) \Delta t^2 + \mathcal{O}(\Delta t^3).
    \label{eq:disc_jac}
\end{equation}
This true discrete flow Jacobian $\bm{J}_{\text{disc}}(\bm{x}, p)$ thus traces out a curved, quadratic path parameterized by $p$. We formalize the limitation this imposes on affine discrete-time DSR models (such as RNNs) in the following theorem.

\begin{theorem}[\textbf{Geometric mismatch in discrete-time affine models}]\label{theorem:theorem_2}
Assume the true continuous-time system has a strictly affine Jacobian dependence on the control parameter $p$ (Eq.~\ref{eq:true_system}). The parameters of a discrete-time DSR model utilizing an affine feature mapping $\bm{J}_{\text{model}}(l) = \hat{\bm{J}}_{0} + l \hat{\bm{J}}_{1}$, where the feature scales as $l = g(p)$ as in Eq.~\ref{eq:g} (or $l_{\text{dyn}} = g(p)$ for models with feature splitting), are mathematically restricted to a one-dimensional straight line on the matrix subspace, i.e., to matrices in $\mathbf{R}^{N\times N}$ fully parameterized by a single scalar. This architectural constraint strictly prevents the model from capturing the true system's quadratic parameter dependence, inducing an irreducible truncation error $\mathcal{O}(p^2 \Delta t^2)$ during extrapolation.
\end{theorem}

\textit{Proof sketch (Full proof in Appx.~\ref{appx:proof_thm2}.)}$\;$
By optimizing the model Jacobian to match the first-order terms of the true discrete-time Jacobian (Eq.~\ref{eq:disc_jac}), the model captures the linear parameter dependence: $\hat{\bm{J}}_{0} \approx \bm{I} + \bm{J}_0 \Delta t + \frac{1}{2}(\bm{J}_0^2 + \dot{\bm{J}}_0) \Delta t^2$ and $\hat{\bm{J}}_{1} \approx \bm{J}_1 \Delta t + \frac{1}{2}(\bm{J}_0\bm{J}_1 + \bm{J}_1\bm{J}_0 + \dot{\bm{J}}_1)\Delta t^2$. However, the affine model lacks higher-order feature terms (e.g., $l^2$ or $l_\mathrm{dyn}^2$). Consequently, it cannot reproduce the quadratic matrix term $\frac{1}{2} p^2 \bm{J}_1^2 \Delta t^2$ required by the true discrete-time system. While this structural omission can be masked during in-domain interpolation because standard ERM only optimizes for local training loss, extrapolating outside the training domain inevitably amplifies the spatial mismatch caused by this missing quadratic curvature, placing a strict bound on the model's structural fidelity.

This mismatch cannot be resolved by the previously introduced solutions of sparsity constraints or feature splitting. However, understanding the structural truncation error allows us to derive a closed-form mathematical diagnostic for the reliable extrapolation range.

\begin{corollary}[Reliable extrapolation range]\label{cor:cor_2_1}
Define the second-order residual matrix
\begin{align}
    \bm{R}_{\text{est}}(l):=\tfrac12(\hat{\bm{J}}_{0}-\bm{I}+l\hat{\bm{J}}_{1})^2,
    \qquad
    \mathcal{I}_\varepsilon:=\{\,l:\|\bm{R}_{\text{est}}(l)\|\le \varepsilon\,\}.
\end{align}
For any tolerance \(\varepsilon>0\), \(\mathcal{I}_\varepsilon\) is the maximum reliable extrapolation set. Outside \(\mathcal{I}_\varepsilon\), the leading \(\mathcal{O}(p^2\Delta t^2)\) term from Theorem~\ref{theorem:theorem_2} exceeds the tolerance level $\varepsilon$.
\end{corollary}

\textit{Derivation.} Matching the $\mathcal{O}(\Delta t)$ terms in the proof of Theorem~\ref{theorem:theorem_2} gives $\hat{\bm{J}}_{0} \approx \bm{I} + \bm{J}_0\Delta t$ and $l\hat{\bm{J}}_{1} \approx p\bm{J}_1\Delta t$, hence $\bm{J}_{\text{cont}}\Delta t \approx \hat{\bm{J}}_{0}-\bm{I}+l\hat{\bm{J}}_{1}$ when the first-order generator is recovered correctly. The leading quadratic contribution from Theorem~\ref{theorem:theorem_2} is $\tfrac12(\bm{J}_{\text{cont}}\Delta t)^2$, which translates to $\bm{R}_{\text{est}}(l)=\tfrac12(\hat{\bm{J}}_{0}-\bm{I}+l\hat{\bm{J}}_{1})^2$. Bounding $\|\bm{R}_{\text{est}}(l)\|\le\varepsilon$ then reduces to a scalar quadratic inequality in $l$, with boundary values $(l_{\min}, l_{\max})$ determined by $\hat{\bm{J}}_{0}$, $\hat{\bm{J}}_{1}$, at which the tolerance is reached.

\section{Conclusions}
\label{sec:discussion_conclusion}

We identified three fundamental structural discrepancies that prevent existing feature-based hierarchical DSR approaches from extrapolating reliably across bifurcations: dense Jacobian coupling where real systems are sparse, entangled positional and dynamical scaling under single-feature parameterizations (as used in recent literature~\cite{yin2021leads, kirchmeyer2022generalizing, brenner2024learning, huh_context-informed_2025}), and geometric mismatches introduced by continuous-to-discrete time mappings.
By introducing L1 sparsity regularization and feature splitting (decoupling positional and dynamical features), we resolve the conflicting scaling requirements and restrict the model to physically plausible parameter dependence. We further derive a closed-form upper bound on the reliable extrapolation range for discrete-time DSR models. With these amendments, extrapolation beyond the training regime and achieving true topological OODG becomes feasible \textit{without explicit knowledge of the true control parameters} \cite{van_tegelen_neural_2025, panahi2024adaptable, kong_machine_2021, kim_inferring_2021}, which we usually do not have for complex real-world systems.

\paragraph{Limitations} We made the assumption that control parameters enter in affine form (or at least in a single functional form) into the physical equations, which might not always be met. Likewise, the sparseness assumption might not be true for all systems. As stated in Theorem \ref{theorem:theorem_2}, there may also be more principled obstacles when moving from continuous to discrete time, which might be harder to resolve and require further research.

\section{Acknowledgments}
\label{sec:acknowledgements}
This work was supported by the German Research Foundation (DFG) through individual grants Du 354/15-1 (\# 502196519) and Du 354/18-1 (\# 567025973), and within Germany’s Excellence Strategy EXC 2181/1 – 390900948 (STRUCTURES). Further support came from the German Ministry for Research, Astronautics, and Technology (BMFTR) through the FEDORA (\# 01EQ2403F) research consortium.

\bibliographystyle{plainnat}
\bibliography{literature}

\newpage
\appendix

\section{Systems with Multiple Parameters and Nonlinear Dependencies}
\label{appx:multidim_parameters}

The analysis in Sect.~\ref{sec:problem-2} and \ref{sec:nonlin_scaling}, which established the structural limitations of affine hierarchical models for the scalar parameter case $p \in \mathbb{R}$, extends directly to systems depending on multiple parameters $\bm{p}=(p_1,\dots,p_n)^T$ of the form
\begin{equation}
    \bm{f}(\bm{x}, \bm{p}) = \bm{f}_0(\bm{x}) + \sum_{i=1}^n p_i \bm{f}_i(\bm{x}),
\end{equation}
where each $p_i$ enters the system multiplicatively, scaling a corresponding vector field component $\bm{f}_i(\bm{x})$. To see this, fix any index $k$ and rewrite the system as
\begin{equation}
    \bm{f}(\bm{x}, \bm{p}) = \tilde{\bm{f}}_k(\bm{x}, \bm{p}_k^{\perp}) + p_k \bm{f}_k(\bm{x}),
\end{equation}
where $\bm{p}_k^{\perp} := (p_1,\dots,p_{k-1},p_{k+1},\dots,p_n)^T$ collects all parameters except $p_k$, and $\tilde{\bm{f}}_k(\bm{x}, \bm{p}_k^{\perp}) := \bm{f}_0(\bm{x}) + \sum_{i \neq k} p_i \bm{f}_i(\bm{x})$ groups all terms independent of $p_k$. The resulting expression is affine in $p_k$, with the remaining terms playing the role of a modified baseline vector field. Since this argument holds for any $k$, all previous derivations apply to each parameter individually, and hence to the full multi-parameter system.

More generally, systems in which parameters appear nonlinearly can be brought into the same affine form by a reparameterization. Consider, for example,
\begin{equation}
    \bm{f}(\bm{x}, \bm{p}) = \bm{f}_0(\bm{x}) + p_1^{\beta}\bm{f}_1(\bm{x}) + p_2 p_3 \bm{f}_2(\bm{x}).
\end{equation}
Defining $q_1 := p_1^{\beta}$ and $q_2 := p_2 p_3$ gives
\begin{equation}
    \bm{f}(\bm{x}, \bm{q}) = \bm{f}_0(\bm{x}) + q_1 \bm{f}_1(\bm{x}) + q_2 \bm{f}_2(\bm{x}),
\end{equation}
which is linear in the transformed parameters $\bm{q}$. The model therefore effectively learns an embedding $\bm{l}(\bm{q})$ rather than $\bm{l}(\bm{p})$. Two important limitations arise in this context, however. First, if the reparameterization mapping $\bm{p} \mapsto \bm{q}$ is non-bijective (e.g., $q_1 = p_1^2$), distinct physical parameter configurations map to identical dynamics, which breaks the unique identifiability of the original control parameters and violates the injectivity assumed in Eq.~\ref{eq:g}. Second, when a single original parameter $p_i$ appears in multiple distinct algebraic forms within the system (e.g., when both $p_1$ and $p_1^2$ enter separately), each such form requires its own feature dimension. This artificially expands the effective parameter space beyond what the true linear parameter dependence requires. As established in Corollary~\ref{cor:cor_1_2}, because standard ERM strictly optimizes for local training loss without structural priors, it leaves these excess dimensions unconstrained, allowing them to retain non-zero weights. Extrapolating outside the training domain can then cause an unbounded growth of the model's spectral radius, potentially introducing spurious bifurcations that are absent from the true system.

\section{Dynamical and positional scaling for complex models}
\label{appx:feature_scalings}

\paragraph{Deeper piecewise-linear networks} In the main text, our structural analysis of affine hierarchical models primarily assumed a minimal single-layer RNN architecture (Eq.~\ref{eq:rnn}) for analytical clarity, and because this form has proven to be efficient for DSR problems \cite{brenner_almost_2024}. Here we show that the fundamental limitation regarding the asymptotic entanglement of positional and dynamical scaling extends to deeper architectures. Consider a single-hidden-layer RNN with piecewise-linear activation function $\varphi$:
\begin{equation}
    \bm{z}_{t+1} = \bm{W}_1(l)\,\varphi\!\left(\bm{W}_2(l)\,\bm{z}_t + \bm{h}_2(l)\right) + \bm{h}_1(l).
    \label{eq:complex_models}
\end{equation} 

Within a single linear subregion of $\varphi$, we have $\varphi(\bm{z}) = \bm{D}\bm{z} + \bm{c}$ for some diagonal matrix $\bm{D}$ and offset $\bm{c}$. In this region, the FP $\bm{z}^*(l)$ can be solved for analytically:
\begin{equation}
    \bm{z}^*(l) = \big( \bm{I} - \bm{W}_1(l)\,\bm{D}\,\bm{W}_2(l) \big)^{-1} \big( \bm{W}_1(l)(\bm{D}\,\bm{h}_2(l) + \bm{c}) + \bm{h}_1(l) \big).
\end{equation}
Because the network matrices each scale affinely with the latent feature $l$, their inverses introduce rational dependence on $l$ in the FP scaling. Thus, $\bm{z}^*(l)$ is a multi-dimensional rational function of $l$, and in general not simply polynomial.

The local Jacobian within the same linear subregion is
\begin{equation}
    \bm{J}(l) = \bm{W}_1(l) \bm{D} \bm{W}_2(l),
\end{equation}
which scales \emph{quadratically} in $l$ due to the product of the two affinely parameterized weight matrices. Positional scaling (how the fixed point moves with $l$) and dynamical scaling (how the Jacobian grows with $l$) therefore belong to different function classes, just as in the minimal architecture analyzed in the main text (Eq.~\ref{eq:rnn}). The structural mismatch and asymptotic entanglement established in Theorem~\ref{theorem:theorem_1} thus persists in single-hidden-layer networks. For even deeper networks, multiplying additional affine weight matrices yields higher-order polynomials in $l$, further compounding this structural mismatch.

The same conclusion holds for continuous-time Neural ODEs with a piecewise-linear hidden layer. For these, the equilibrium condition $\dot{\bm{z}}=\bm{0}$ yields
\begin{equation}
    \bm{z}^*(l) = -\big( \bm{W}_1(l) \bm{D} \bm{W}_2(l) \big)^{-1} \big( \bm{W}_1(l)(\bm{D} \bm{h}_2(l) + \bm{c}) + \bm{h}_1(l) \big)
\end{equation}
which, by the same reasoning as above, is again a multi-dimensional rational function of $l$. The fundamental entanglement of positional and dynamical scaling is therefore an architectural bottleneck shared by both discrete- and continuous-time hierarchical DSR frameworks with hidden layers.

\paragraph{Hidden-layer continuous-time networks with bounded smooth activations}
The piecewise-linear case above has the advantage that FPs can be computed in closed form. For smooth nonlinear activation functions such as $\tanh$ or \texttt{sigmoid}, this is not necessarily the case anymore. Nevertheless, the same structural conclusions about the entanglement of positional and dynamical scaling carry over, as we show now.

Consider a single-hidden-layer Neural ODE with vector field $\bm{f}_{\bm{\theta}}(\bm{z}, l) = \bm{W}_1(l)\bm{\psi}(\bm{W}_2(l)\bm{z} + \bm{h}_2(l)) + \bm{h}_1(l)$, where $\bm{\psi}\in C^1$ is a globally smooth, bounded activation function 
with $\|\bm{\psi}\|_\infty<\infty$ and $\|\bm{\psi}'\|_\infty<\infty$. Even if we can no longer solve $\bm{f}_{\bm{\theta}}(\bm{z}, l) = \bm{0}$ analytically, the implicit function theorem yields a well-defined expression for how the FP changes with the feature $l$, provided the model Jacobian
\begin{equation}
    \bm{J}_{\text{model}}(l) = \bm{W}_1(l)\,\bm{D}(\bm{z}^*(l), l)\,\bm{W}_2(l), 
    \quad 
    \bm{D}(\bm{z}^*(l),l) := \mathrm{diag}\bigl(\bm{\psi}'(\bm{W}_2(l)\bm{z}^*(l) + \bm{h}_2(l))\bigr),
\end{equation}
is invertible on the extrapolation region (an analogue of the ``no eigenvalue at zero'' assumption used in Theorem~\ref{theorem:theorem_1}). Differentiating
$\bm{f}_{\bm{\theta}}(\bm{z}^*(l), l) = \bm{0}$ with respect to $l$ and rearranging gives
\begin{equation}
    \frac{d\bm{z}^*}{dl}
    = -\bm{J}_{\text{model}}(l)^{-1}
      \frac{\partial \bm{f}_{\bm{\theta}}}{\partial l}.
    \label{eq:appxD_smooth_dzstardl}
\end{equation}

To bound the growth of $\bm{z}^*(l)$ as commanded by this equation, we utilize two observations: First, because $\bm{\psi}$ and $\bm{\psi}'$ are uniformly bounded, $\bm{D}(\bm{z}^*(l),l)$ is uniformly bounded in operator norm, and $\partial \bm{f}_{\bm{\theta}}/\partial l$ is affine in the feature weights and therefore at most polynomially bounded in $\|\bm{z}^*\|$ (constant for bounded $\bm{\psi}$, degree-$1$ for the polynomially bounded case). Second, applying the same singular-value lower bound logic used in the proof of Corollary~\ref{cor:cor_1_1} – specifically, $\sigma_{\min}\bigl(\bm{W}_1(l)\bm{D}(\bm{z}^*(l),l)\bm{W}_2(l)\bigr) \ge c\,|l|^2$ – Eq.~\ref{eq:appxD_smooth_dzstardl} yields the following differential inequality
\begin{equation}
    \bigl\|d\bm{z}^*/dl\bigr\|_2 \le \tfrac{C_0 + C_1\|\bm{z}^*\|^q}{c\,|l|^2}
\end{equation}
for constants $C_0, C_1 \ge 0$ and degree $q \ge 0$ determined by the activation function. This is the same fundamental differential inequality underlying the proofs of Theorem~\ref{theorem:theorem_1} and Corollary~\ref{cor:cor_1_1}, but with an even faster decay rate due to the depth of the network. Integrating it with respect to $l$ reveals a strict geometric limit: because the integral of $1/|l|^2$ converges to a finite constant as $l \to \infty$, the fixed point cannot grow unboundedly. Instead, it is strictly bounded by a constant, $\|\bm{z}^*(l)\|_2 = \mathcal{O}(1)$, for both globally bounded and polynomially bounded activation functions. 

While we no longer have an explicit rational form for $\bm{z}^*(l)$ as in the simple no-hidden-layer case, the fundamental entanglement persists: the model's fixed point is structurally forced into an $\mathcal{O}(1)$ asymptotic bound, whereas the true physical system requires independent algebraic scaling $\Theta(p^m)$. This confirms that the architectural bottleneck preventing independent scaling—and thus the conclusion of Theorem~\ref{theorem:theorem_1}—extends to hidden-layer Neural ODEs with smooth activations as well.

\section{Dependence of RNN flow map on features}
\label{appx:flow_map_dependence}

In Eq.~\ref{eq:fp_RNN} we derived how FPs of the discrete-time RNN flow map change as a latent feature $l$ is varied. This rate of change $d\bm{z}^*/dl$ describes how the location $\bm{z}^*$ of a FP moves in state space as the feature $l$ is altered, and depends on the product of the inverse model Jacobian and the partial derivative of the model flow map, $\partial \bm{F}_{\bm{\theta}}/\partial l$.

For the minimal architecture utilized in the main text, defined as $\bm{F}_{\bm{\theta}}(\bm{z}) = \bm{W}(l) \bm{\psi}(\bm{z}) + \bm{h}(l)$ (Eq.~\ref{eq:rnn}) with affinely parameterized weights $\bm{W}(l) = \bm{W}_c + l \bm{W}_v$ and biases $\bm{h}(l) = \bm{h}_c + l \bm{h}_v$, this partial derivative evaluated at the FP $\bm{z}^*$ simplifies to:
\begin{equation}
    \frac{\partial \bm{F}_{\bm{\theta}}}{\partial l} = \bm{W}_v \bm{\psi}(\bm{z}^*) + \bm{h}_v
\end{equation}

More generally, for the extended, single-hidden-layer architecture introduced in Appx.~\ref{appx:feature_scalings}, which utilizes a piecewise-linear activation function $\varphi$, the partial derivative w.r.t. $l$ expands into a more complex form. Applying the chain rule to the extended flow map (Eq.~\ref{eq:complex_models}) yields:
\begin{equation}
\begin{aligned}
    \frac{\partial \bm{F}_{\bm{\theta}}}{\partial l}
    &= \bm{W}_{1,v}\, \varphi(\bm{W}_2(l) \bm{z}^* + \bm{h}_2(l)) \\
    &\quad + \bm{W}_1(l) \text{diag}(\varphi'(\bm{W}_2(l) \bm{z}^* + \bm{h}_2(l)))
    \left( \bm{W}_{2,v} \bm{z}^* + \bm{h}_{2,v} \right) \\
    &\quad + \bm{h}_{1,v}
\end{aligned}
\end{equation}

Crucially, in both the simple and the more complex single-hidden-layer architecture, the variation inducing weights (e.g., $\bm{W}_v, \bm{h}_v$ or $\bm{W}_{1,v}, \bm{W}_{2,v}$, etc.) scale strictly linearly with $l$, and the derivative of the piecewise-linear activation, $\varphi'$, evaluates component-wise to either $0$ or $1$. As a result, $\partial \bm{F}_{\bm{\theta}}/\partial l$ is either constant or affine in $l$ within each linear subregion. This limits the polynomial degree of the positional shift ($d\bm{z}^*/dl \propto \bm{J}(l)^{-1} \partial \bm{F}_{\bm{\theta}}/\partial l$), ensuring it exceeds the degree of the inverse Jacobian by at most one. Consequently, unconstrained affine feature mappings do not provide enough flexibility to capture the \emph{independent} positional and dynamical scaling required for topological OODG.

\section{Architectures of the used models}
\label{appx:used_models}

\subsection{Discrete-time model}
\label{appx:disc_model}
For our experiments we employed the shallow piecewise-linear RNN (shPLRNN; \cite{hess_generalized_2023}), which simply adds a linear term to Eq.~\ref{eq:complex_models} (Appx.~\ref{appx:feature_scalings}), thereby improving its practical performance by stabilizing loss gradients \cite{schmidt_identifying_2021}. It follows the state transition equation
\begin{equation}
    \bm{z}_{t+1} = \bm{A}(l)\bm{z}_t 
    + \bm{W}_1(l)\,\varphi\!\left(\bm{W}_2(l)\bm{z}_t + \bm{h}_2(l)\right)
    + \bm{h}_1(l),
\end{equation}
where $\bm{z}_t \in \mathbb{R}^M$ denotes the latent state at time $t$, and $\varphi$ is the $\mathrm{ReLU}$ activation function. The matrix $\bm{A}(l) \in \mathbb{R}^{M \times M}$ is diagonal, i.e.\ $\bm{A}(l) = \mathrm{diag}(\bm{a}(l))$. The weight matrices are $\bm{W}_1(l) \in \mathbb{R}^{M \times H}$ and $\bm{W}_2(l) \in \mathbb{R}^{H \times M}$, with corresponding biases $\bm{h}_1(l) \in \mathbb{R}^M$ and $\bm{h}_2(l) \in \mathbb{R}^H$. Here, $M$ denotes the latent dimensionality and $H$ the hidden dimensionality.

The latent states are related to the observations via
\begin{equation}
    \bm{x}_t = g\!\left(\bm{z}_t\right),
    \label{eq:observation_model_of_used_models}
\end{equation}
where a linear mapping $g(\bm{z}_t) = \bm{E}\bm{z}_t$ was used here.

Without feature splitting, all parameters are constructed according to Eq.~\ref{eq:affine_map} in Section~\ref{sec:hierarchical_dsr}, using a single feature $l$:
\begin{equation}
    \bm{A}(l) = \bm{A}_c + l\,\bm{A}_v, 
    \qquad 
    \bm{W}_1(l) = \bm{W}_{1,c} + l\,\bm{W}_{1,v}, 
    \qquad \dots
\end{equation}

Likewise, with feature splitting, one feature $l_{\text{dyn}}$ is used for all parameters that enter the Jacobian and thus contribute to the local growth rates, while another feature $l_{\text{pos}}$ is used for parameters only translating the flow without changing its local growth rates (Eq.~\ref{eq:criteria_dyn_pos}). Thus, the positional feature only affects the first bias term $\bm{h}_1$, while all other parameters depend on the dynamical feature:
\begin{equation}
    \bm{h}_2\!\left(l_{\text{pos}}\right) = \bm{h}_{2,c} + l_{\text{pos}}\,\bm{h}_{2,v}, 
    \qquad 
    \bm{A}\!\left(l_{\text{dyn}}\right) = \bm{A}_c + l_{\text{dyn}}\,\bm{A}_v, 
    \qquad \dots
\end{equation}

\subsection{Continuous-time}
\label{appx:used_models_cont}
The continuous-time model used in our empirical evaluations directly parameterizes the vector field as
\begin{equation}
    \dot{\bm{z}}_t = \bm{W}_1(l)\bm{\psi}(\bm{W}_2(l) \bm{z}_t + \bm{h}_2(l)) + \bm{h}_1(l)
\end{equation}
where $\bm{z}_t \in \mathbb{R}^M$ denotes the latent state at time $t$, and $\bm{\psi}$ is the $\tanh$ activation function. The weight matrices are $\bm{W}_1(l) \in \mathbb{R}^{M \times H}$ and $\bm{W}_2(l) \in \mathbb{R}^{H \times M}$, with the corresponding biases $\bm{h}_1(l) \in \mathbb{R}^M$ and $\bm{h}_2(l) \in \mathbb{R}^H$. As in the discrete-time case above, $M$ and $H$ denote latent and hidden dimensionality, respectively. As above, latent states are related to the observations by Eq.~\ref{eq:observation_model_of_used_models}.

Scaling with features $l$ was implemented the same way as for the discrete-time models (\ref{appx:disc_model}), with the distinction between dynamical and positional parameters given by Eq.~\ref{eq:criteria_dyn_pos}.

\subsection{Hyperparameters \& Compute Resources}

For the experiments on the Selkov system, we used $M=2$ and $H=128$ for both the discrete-time and continuous-time setups, and trained each model on ground-truth sequences of up to $20$ steps with generalized teacher forcing \citep{hess_generalized_2023}. Training times were approximately $10\,$min per shPLRNN model (on a CPU of type AMD EPYC 7713 with 64-cores; with multiple models trained in parallel) and about $10\,$min per Neural ODE model (on a GPU of type NVIDIA GeForce RTX 2080 Ti where roughly 2 GB were utilized per model; again, multiple models were trained in parallel).

For the reported results on the Lorenz-63 system with the continuous-time setup, we used $M=12$ and $H=64$, ground-truth sequences of up to $48$ steps and generalized teacher forcing \citep{hess_generalized_2023}. On the same GPU setup as for the Selkov system, approximately $30\,$min were required for training and up to 4 GB of GPU memory per model.

All continuous-time models were optimized using the Adam optimizer~\citep{kingma_adam_2015}, and all discrete-time models with RAdam~\citep{liu_variance_2020}, both as implemented in \texttt{pytorch}~\citep{paszke2019pytorch}.

\section{Dynamical systems used for benchmarking}
For each benchmark data set, we used $J_{\text{ID}}$ different control parameter values as training domain. For each system within this domain, we generated $K$ trajectories of length $T_\text{train}$ for training, and an additional $K$ trajectories of length $T_\text{test}$ for in-domain evaluation. To evaluate out-of-domain (OOD) generalization, we also generated $K$ trajectories of length $T_\text{test}$ each for $J_{\text{OOD}}$ additional systems located outside the training domain. All data were standardized per dimension. To prevent data leakage, the mean and variance were computed exclusively across the training trajectories (spanning all $J_{\text{ID}}$ training systems and time points) and subsequently applied to scale the training, in-domain test, and OOD test sets.

\subsection{Lorenz-63 model}
\label{appx:lorenz}
The Lorenz-63 model was originally introduced by \citet{lorenz_deterministic_1963} as a simplified model of atmospheric convection. It was one of the earliest examples illustrating that a low-dimensional deterministic ODE with simple nonlinearities (low-order multinomials) can generate chaotic behavior:
\begin{align}
    \dot x &= \sigma (y - x) \\
    \dot y &= x (\rho - z) - y \\
    \dot z &= xy - \beta z
\end{align}
We took $\sigma = 10$ and $\beta = 8/3$, and varied the control parameter $\rho$ across $16$ uniformly spaced values in the interval $[5, 22.5]$, where the system converges to one of two FPs, depending on initial condition.

For each parameter setting, we integrated $10$ training trajectories up to time $T=50$ in time steps of $\Delta t = 0.01$ from random uniformly distributed initial conditions $x_0 \sim \mathcal{U}(-20, 20)$, $y_0 \sim \mathcal{U}(-25, 25)$ and $z_0 \sim \mathcal{U}(0, 35)$. As noted above, all training trajectories were standardized dimension-wise.

To test trained models, we created two additional sets of trajectories of length $T=200$. The in-domain test set reuses the same values for $\rho$ as in training, while for the OOD test set we added $30$ systems with values of $\rho \in [22.5, 35]$, allowing us to check for generalization to unseen types of dynamics such as the transition to chaos at $\rho\approx24.06$.

\subsection{Selkov model}
\label{appx:selkov}
The two-dimensional Selkov model \citep{selkov1968self}, which transitions from a stable spiral point to a stable limit cycle through a super-critical Hopf bifurcation, is given by
\begin{align}
    \dot x &= - x + a y + x^2 y \\
    \dot y &= b - a y - x^2 y \ .
\end{align}
For training, we kept $a=0.1$ constant and varied $b$ across $15$ uniformly spaced values in the interval $[0.8, 1.1]$ ($\Delta b=0.02$), all of which correspond to FP dynamics. Per setting, $10$ training trajectories were integrated for $T=5$ at step size $\Delta t=0.1$, starting from uniformly distributed initial conditions $(x, y) \sim \mathcal{U}(-1, 3)$. The initial condition and integration period were deliberately chosen to yield short \textit{transient} trajectories, covering mainly the period of convergence to the fixed points (as the type of fixed point is otherwise impossible to learn from just the stationary portion in the absence of noise). All training trajectories were standardized per dimension.

For both test sets (i.e., in-domain and OOD) the integration time was extended to $T=500$. The OOD test set was created using $15$ values of $b \in [0.3, 0.8]$, covering both fixed point and limit cycle dynamics.

\section{Feature extrapolation strategy}
\label{appx:feature_extrap}

After training a hierarchical model on $J_\text{ID}$ systems with control
parameters $\{p^{(j)}\}_{j=1}^{J_\text{ID}}$ and inferred latent features
$\{l^{(j)}\}_{j=1}^{J_\text{ID}} \subset \mathbb{R}^{L}$, predicting the dynamics for unseen control parameters $\{p^{(j')}\}_{j'=1}^{J_\text{OOD}}$ requires an estimate of the corresponding latent features $\{l^{(j')}\}_{j'=1}^{J_\text{OOD}}$. As discussed in Section~\ref{sec:problem-2}, feature splitting allows $l_\text{dyn}$ and $l_\text{pos}$ to assume independent algebraic relationships with $p$ (e.g.\ $l_\text{dyn}\propto
p$ while $l_\text{pos}\propto\sqrt{p-p_\text{crit}}$ near a Hopf
bifurcation). Rather than committing to a fixed parametric form
\textit{a priori}, we fit each feature (respectively feature dimension if $L>1$) independently to $p$ and select the most appropriate functional class via cross-validation on the training domain.

\paragraph{Candidate models} For every feature dimension, we consider three low-parameter functional forms that cover the scaling
behaviors expected from the analysis in Section~\ref{sec:problem-2}:
\begin{align}
\text{(Power law)}\quad   l_{\text{PL}}(p) &= a\,(p-b)^{d} + c, \quad b<p_\text{min}-\eta\,(p_\text{max}-p_\text{min}), \\
\text{(Signed power law)}\quad  l_{\text{SPL}}(p) &= a\,\mathrm{sign}(p-b)\,|p-b|^{d} + c, \\
\text{(Parabola)}\quad l_{\text{Parab}}(p) &= a\,(p-b)^{2} + c,
\end{align}
with exponent $d\!\in\![0.5, 1.5]$. The power law models algebraic scaling with a fractional exponent, which is the natural functional form for the square-root-type behavior characteristic of multiple normal forms near local bifurcations. For any fit with $l_{\text{PL}}(p)$, we require that $b < p_{\min} - \eta(p_{\max} - p_{\min})$, where $p_{\min}$ and $p_{\max}$ are the smallest and largest control parameter values (across both in-domain and OOD sets), respectively, and $\eta=1$ regulates the margin. This ensures that all points to which we seek to extrapolate lie safely within the real-valued domain of the function (avoiding complex roots where $p - b < 0$), and that the highly nonlinear region near $p=b$ is avoided. Additionally, we allow fits with a signed power law, which in contrast has full support, as well as a sign change similar to an inflection point at $p=b$. The parabolic fit is incorporated to cover non-monotonic trends which were observed to occur for higher-dimensional feature vectors.

\paragraph{Robust regression} Each candidate function is fit by nonlinear least squares with a soft-$\ell_1$ loss (\texttt{scipy.optimize.least\_squares}), which downweighs points with large residuals and thereby prevents a few aberrant features from dominating the fit. After an initial fit, points with standardized residual $|r_j|/\mathrm{RMSE}>k$ ($k=2.5$) are flagged as outliers and the model is refit on the remaining points. We cap the fraction of points dropped at $5\%$.

\begin{figure}[h]
    \centering
    \includegraphics[width=\linewidth]{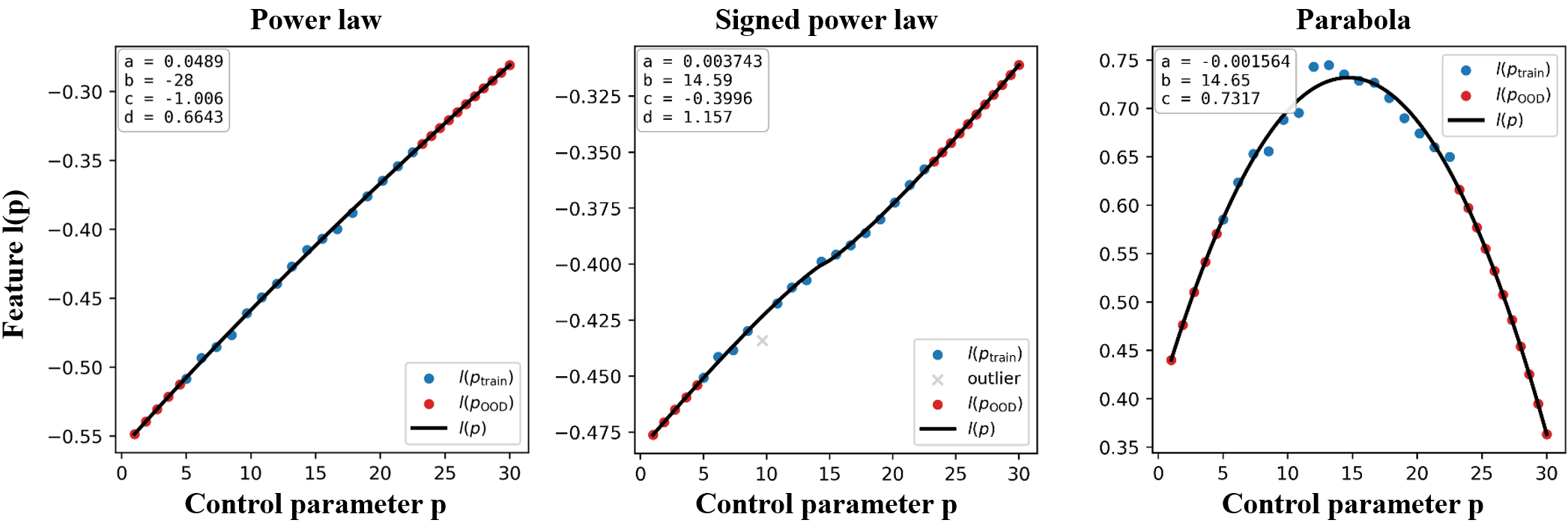}
    \caption{Empirical examples of extrapolation fits obtained from the three candidate models: power law (left), signed power law (center), and parabola (right). Blue markers correspond to in-domain training data, red markers to OOD samples, gray markers indicate discarded outliers, and the black curve represents the fitted function.}
    \label{fig:extrapolation_fits}
\end{figure}

\paragraph{Model selection by cross-validation} The three candidate function classes are compared by cross-validation: across $S=5$ splits, we hold out a fraction $\tau=0.2$ of all systems, fit each candidate function (with the same outlier handling) on the remaining systems, and evaluate the mean absolute error on the held-out systems,
\begin{equation}
\mathrm{MAE}_\text{CV}^{(m)} = \frac{1}{S}\sum_{s=1}^{S}\frac{1}{|J_s|}\sum_{j\in J_s}\bigl|l^{(m,s)}(p^{(j)})-l^{(j)}\bigr|,
\end{equation}
for $m\in\{\text{PL},\text{SPL},\text{Parab}\}$, where $J_s$ is the set of indices of the held-out systems in split $s$, and $|J_s|$ denotes the size of this set. We use MAE rather than MSE to prevent single, badly captured systems from dominating the comparison. The winning candidate function class $m^{*}=\arg\min_{m}\mathrm{MAE}_\text{CV}^{(m)}$ is then re-fit on the full set of $J_{\text{ID}}$ systems (excluding outliers) to produce the final extrapolated features at the requested OOD parameters $\{p^{(j')}\}_{j'=1}^{J_\text{OOD}}$.

\section{Evaluation metric}
\label{appx:evaluation}
To quantify reconstruction quality, we compare the bifurcation diagrams of the true and reconstructed systems by treating them as empirical distributions and measuring the discrepancy between them.
In practice, we obtain bifurcation diagrams by evaluating predicted and true test trajectories for each system (in-domain and OOD), cutting off any transients and extracting local minima and maxima in the $x_1$ dimension. This yields two empirical samples
\begin{equation}
A = \{a_i\}_{i=1}^{n}, \qquad B = \{b_j\}_{j=1}^{m},
\end{equation}
from the true and reconstructed systems at each control parameter configuration, which we treat as one-dimensional empirical distributions on $\mathbb{R}$. For trajectories converging to a FP, the last $500$ time points from the trajectory are used as the sample from the distribution. Crucially, the Wasserstein-1 distance is well-defined even when comparing singular distributions (e.g., isolated points) to continuous distributions (e.g., chaotic attractors), as it relies on optimal transport cost rather than overlapping support, making it suited for evaluating topology changes across bifurcations.

We quantify the discrepancy between these distributions using the Wasserstein-1 distance which evaluates the optimal transport cost for moving one distribution into another \cite{vallender1974calculation, villani2009optimal}. For probability measures $\mu, \nu$ on $\mathbb{R}$ with finite first moments, it is defined as
\begin{equation}
    W_1(\mu,\nu) = \int_0^1 \bigl|Q_\mu(q) - Q_\nu(q)\bigr| \, dq,
\end{equation}
where $Q_\mu, Q_\nu$ are the quantile functions of $\mu$ and $\nu$ \cite{villani2009optimal}.

This formulation motivates a simple empirical estimator. Let $A$ and $B$ be sorted in ascending order and let $\hat{Q}_A, \hat{Q}_B$ denote their empirical quantile functions, obtained by linear interpolation between the sample points. Evaluating both on a common uniform grid $\{q_k = (k-1)/(K-1)\}_{k=1}^{K}$ with $K=\max(n,m)$, we estimate
\begin{equation}
    \widehat{W}_1(A,B) = \frac{1}{K}\sum_{k=1}^{K}
                    \bigl|\hat{Q}_A(q_k) - \hat{Q}_B(q_k)\bigr|.
\end{equation}
Choosing $K=\max(n,m)$ ensures that the larger of the two samples is used at full resolution while the smaller is up-sampled by linear interpolation of its quantile function; non-finite values in either sample are removed before sorting.

\section{Additional results}
\label{appx:additional_results}
Here we examine how feature splitting and low-rank-\&-sparsity-regularization interact when combined. For this, two benchmark systems with different extrapolation regimes were examined. We find that for the Selkov system (Fig.~\ref{fig:wasserstein_selkov_node}) feature splitting is the dominant factor: it improves both in-domain stability near the training boundaries and OOD performance across the full extrapolation range. Adding regularization yields no consistent additional gain. On the Lorenz-63 system extrapolated into the chaotic regime (Fig.~\ref{fig:l1_lowrank_splitting}), the picture reverses: feature splitting alone is insufficient to prevent the full-rank model from diverging, and the combination of regularization and splitting produces the lowest OOD error. Taken together, these results suggest that the two interventions address complementary failure modes (see Section~\ref{sec:problems}).
\begin{figure}[ht!]
    \centering
    \includegraphics[width=\linewidth]{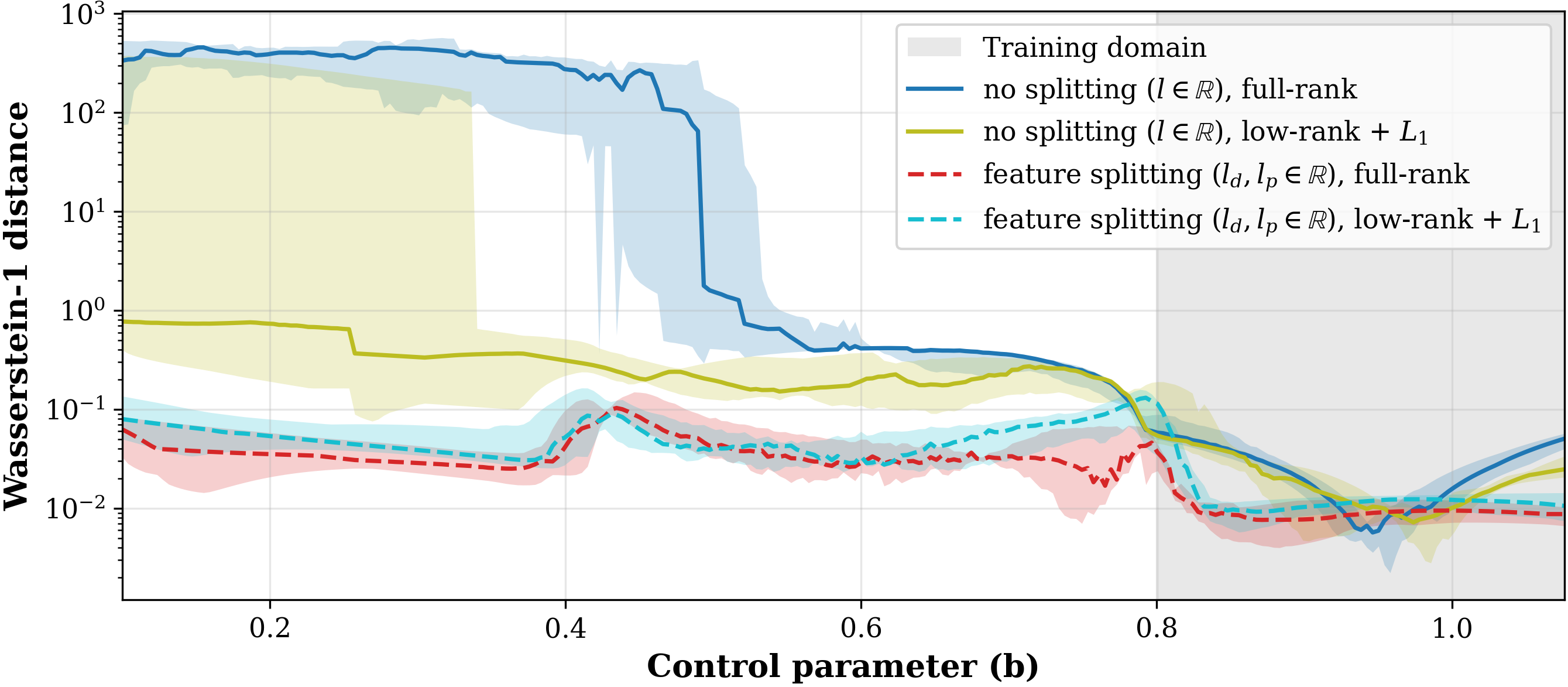}
    \caption{
    Selkov system: combining feature splitting and low-rank-\&-sparsity-regularization. Wasserstein-1 distances between true and reconstructed bifurcation diagrams for hierarchical Neural ODE models with no, one, or both amendments to the hierarchical model. Training data lies in the fixed-point regime; OOD extrapolation covers the limit-cycle regime and the second Hopf bifurcation back to FPs. Solid and dashed lines show medians over 10 runs; envelopes indicate the interquartile range. The gray shaded region indicates the training domain.
    }
    \label{fig:wasserstein_selkov_node}
\end{figure}

\begin{figure}[ht!]
    \centering
    \includegraphics[width=\linewidth]{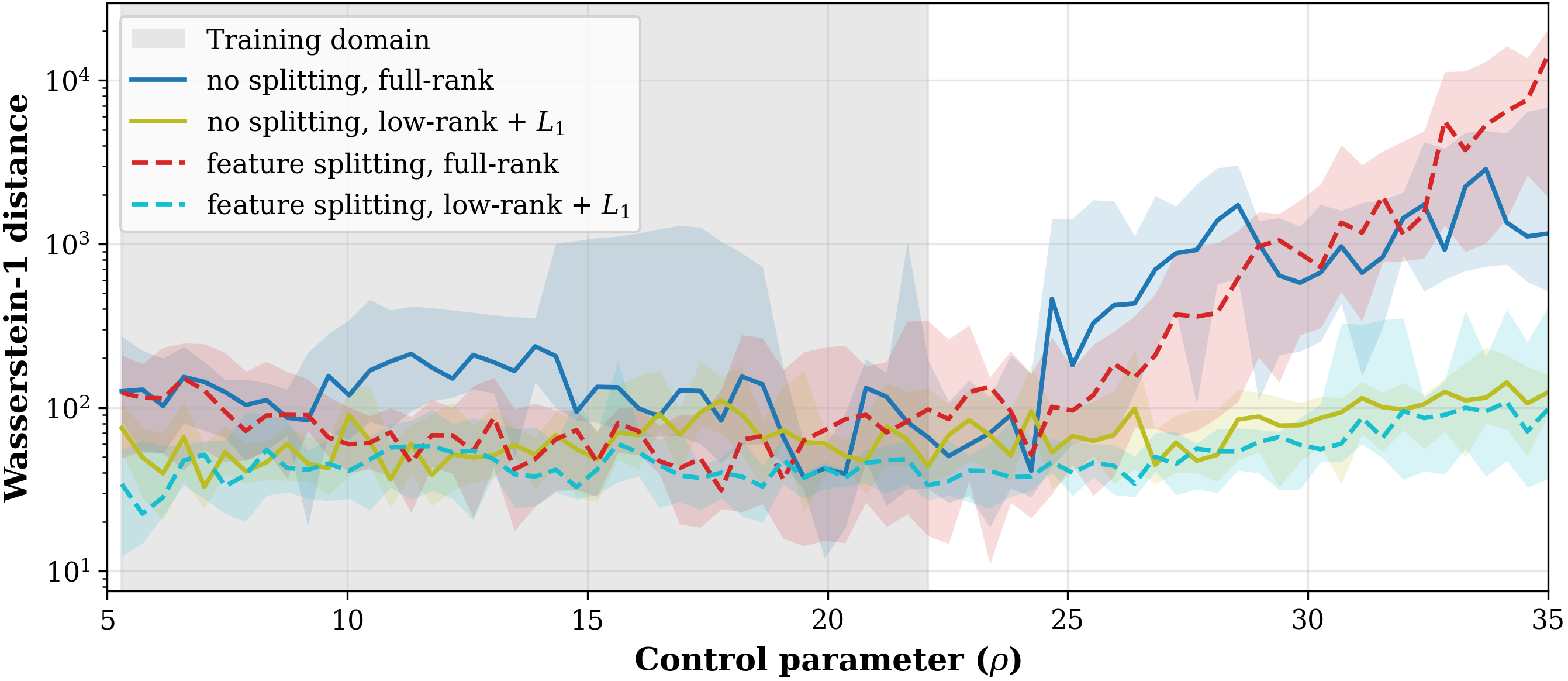}
    \caption{Lorenz-63 system: combining feature splitting and low-rank-\&-sparsity-regularization. Same setup as Fig.~\ref{fig:wasserstein_selkov_node}, with training data in the two-fixed-point regime and OOD extrapolation across the transition to chaos. Solid and dashed lines show medians over 10 runs; envelopes indicate the interquartile range. The gray shaded region indicates the training domain.
    }
    \label{fig:l1_lowrank_splitting}
\end{figure}

Additionally, a comparison of shPLRNN models with and without low-rank-\&-sparsity regularization on the Lorenz-63 DS is presented in Fig.~\ref{fig:bif_medians_lorenz_shplrnn}, which provides the bifurcation diagram and a sampled trajectory of $10,000$ steps, with median Wasserstein-1 distance over $n=10$ trained models per setting. The regularized model recovers the chaotic attractor at $\rho=29$, well outside the training domain, whereas the unconstrained model breaks down and converges to a FP, highlighting the need for the low-rank-\&-sparsity regularization.

\begin{figure}[ht!]
    \centering
    \includegraphics[width=\linewidth]{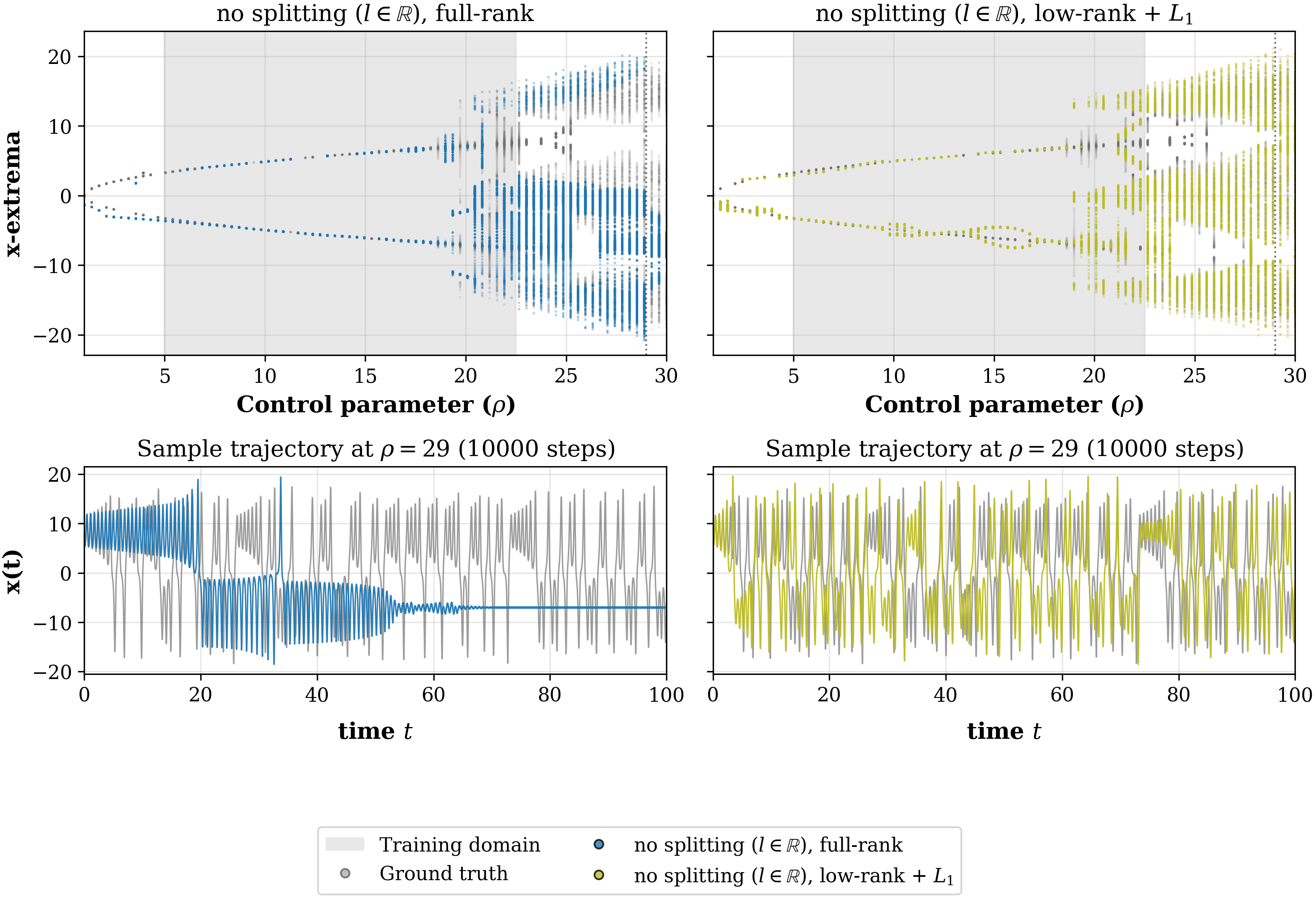}
    \caption{shPLRNN models trained on Lorenz-63, without (left) and with (right) low-rank-\&-sparsity regularization. \textbf{Top:} reconstructed bifurcation diagrams. \textbf{Bottom:} sample trajectories at $\rho=29$. The regularized model captures the chaotic regime, while the unconstrained model fails to reconstruct it.}
    \label{fig:bif_medians_lorenz_shplrnn}
\end{figure}

\section{Proofs}
\subsection{Proof of Proposition~\ref{prop:prop_1}}
\label{appx:proof_prop1}

\textit{Proof.} Let $\mathcal{S} \subset \mathbb{R}^{N \times N}$ denote the $k$-dimensional subspace of matrices with the true physical sparsity pattern. By assumption, the true parameter-dependent part of the Jacobian satisfies $\bm{J}_1 \in \mathcal{S}$. The model uses an affine Jacobian parameterization of the form $$ \bm{J}_{\mathrm{model}}(l)=\hat{\bm{J}}_0 + l\,\hat{\bm{J}}_1, $$ where $l$ is the learned latent feature.

To demonstrate that this extrapolation failure is a structural artifact of the affine parameterization rather than a consequence of poor training, we consider an ideal in-domain scenario. Suppose the model perfectly captures the true parameter dependence during training, such that the learned feature scales strictly linearly with the control parameter, $l = \alpha p$ for some scalar $\alpha \neq 0$.

We now decompose the learned variational matrix $\hat{\bm{J}}_1$ into its components parallel and orthogonal to $\mathcal{S}$ with respect to the Frobenius inner product:
$$ \hat{\bm{J}}_1 = \hat{\bm{J}}_1^{\parallel} + \hat{\bm{J}}_1^{\perp},
\qquad \hat{\bm{J}}_1^{\parallel} \in \mathcal{S},
\quad \hat{\bm{J}}_1^{\perp} \in \mathcal{S}^{\perp}. $$
Here, $\hat{\bm{J}}_1^{\parallel}$ is the component aligned with the true physical parameter dependence, whereas $\hat{\bm{J}}_1^{\perp}$ lies in directions along which there is no variation in the true system.

Because the feature-coupling weights $\bm{\theta}_v$ in Eq.~\ref{eq:affine_map} are full-rank and no structural prior or sparsity-promoting regularization is imposed, empirical risk minimization (ERM) only constrains the projection of $\hat{\bm{J}}_1$ onto the true subspace $\mathcal{S}$. The orthogonal component $\hat{\bm{J}}_1^{\perp}$, which lives in the $(N^2-k)$-dimensional complement $\mathcal{S}^{\perp}$, therefore remains under-determined. Hence, generically, the learned solution may not satisfy $\hat{\bm{J}}_1^{\perp}=0$; we write
$$ \|\hat{\bm{J}}_1^{\perp}\|_F = \epsilon > 0. $$

Consider now the Jacobian extrapolation error outside the training domain, 
$$ \bm{E}(p) := \bm{J}_{\mathrm{model}}(\alpha p) - \bm{J}_{\mathrm{true}}(p). $$
Even under the favorable assumption of perfect in-domain recovery of the physically meaningful part, namely
$$ \hat{\bm{J}}_0 = \bm{J}_0,
\qquad \alpha \hat{\bm{J}}_1^{\parallel} = \bm{J}_1,$$
the remaining error is exactly the orthogonal component:
$$ \bm{E}(p)
= \bigl(\hat{\bm{J}}_0 + \alpha p (\hat{\bm{J}}_1^{\parallel} + \hat{\bm{J}}_1^{\perp})\bigr)
- \bigl(\bm{J}_0 + p\bm{J}_1\bigr)
= \alpha p\,\hat{\bm{J}}_1^{\perp}. $$
Thus, the extrapolation error grows linearly in $|p|$ whenever $\hat{\bm{J}}_1^{\perp} \neq 0$, and this growth occurs entirely in directions orthogonal to the true sparse parameter subspace $\mathcal{S}$.

To understand how this growing matrix error corrupts the reconstructed dynamics, we must evaluate the system's eigenvalue spectrum, which dictates local stability and the occurrence of local bifurcations. We require a mechanism to translate the spatial matrix error into spectral drift. The Bauer--Fike theorem \cite{bauer1960norms} provides this mathematical bridge, bounding how far the eigenvalues of a perturbed matrix can deviate from their true values based on the magnitude of the perturbation. Assuming that the true Jacobian $\bm{J}_{\mathrm{true}}(p)$ is diagonalizable, i.e.,
$$ \bm{J}_{\mathrm{true}}(p) = \bm{V}\bm{\Lambda}\bm{V}^{-1}, $$
the theorem dictates that for any eigenvalue $\tilde{\lambda}$ of the extrapolated model Jacobian $\bm{J}_{\mathrm{model}}(\alpha p)$, there exists an eigenvalue $\lambda$ of $\bm{J}_{\mathrm{true}}(p)$ such that
$$ |\tilde{\lambda} - \lambda| \leq
\kappa(\bm{V}) \|\bm{E}(p)\|_2 \leq
\kappa(\bm{V})\, |\alpha p|\, \|\hat{\bm{J}}_1^{\perp}\|_2, $$
where $\kappa(\bm{V})=\|\bm{V}\|_2 \|\bm{V}^{-1}\|_2$ is the condition number of the true eigenvector matrix. Using $\|\hat{\bm{J}}_1^{\perp}\|_2 \leq \|\hat{\bm{J}}_1^{\perp}\|_F = \epsilon$, we obtain $$ |\tilde{\lambda} - \lambda| \leq \kappa(\bm{V})\, |\alpha p|\, \epsilon. $$

Therefore, as soon as $\epsilon > 0$, the spectral deviation is allowed to grow as $\mathcal{O}(|p|)$ outside of the training domain. Because this unconstrained spectral drift acts in directions completely decoupled from the true physical mechanism, it can push eigenvalues across the imaginary axis or unit circle in a non-physical way, thereby inducing spurious local bifurcations. $\square$

\subsection{Proof of Theorem~\ref{theorem:theorem_1}}
\label{appx:proof_thm1}

\textit{Proof.} The fundamental bottleneck of affine feature mappings is that the model's fixed point (FP), $\bm{z}^*(l)$, and its local Jacobian, $\bm{J}_{\text{model}}(l)$, are structurally coupled through the same network parameters. The FP is determined by the equilibrium condition $\bm{W}(l)\bm{\psi}(\bm{z}^*) + \bm{h}(l) = \bm{z}^*$ for the discrete-time RNN (Eq.~\ref{eq:fp_RNN}) or $\bm{W}(l)\bm{\psi}(\bm{z}^*) + \bm{h}(l) = \bm{0}$ for the continuous-time Neural ODE (Eq.~\ref{eq:fp_NODE}). The local Jacobian at this FP is $\bm{J}_{\text{model}}(l) = \bm{W}(l)\text{diag}(\bm{\psi}'(\bm{z}^*))$ (Eq.~\ref{eq:J_RNN_NODE}). 

Because both FP location and Jacobian 
share this exact affine dependence on the latent feature (via $\bm{W}(l) = \bm{W}_c + l \bm{W}_v$), requiring the model to correctly extrapolate the true system's dynamics turns this shared parameterization into a strict algebraic constraint. Assume there exists a continuous feature mapping $l(p)$ that captures the true system's asymptotic dynamical scaling, meaning its growth scales exactly as $\|\bm{J}_{\text{model}}(l(p))\| = \Theta(p^n)$. To achieve this exact polynomial growth in the Jacobian, the affine structure of the weights rigidly forces the latent feature to scale as $|l(p)| = \Theta(p^n)$.

Consider first the case where the activation function $\bm{\psi}$ is globally bounded (e.g., $\tanh$). As the control parameter diverges ($p \to \infty$), the output of the activation function evaluated at the FP, $\bm{\psi}(\bm{z}^*)$, remains strictly bounded. The positional scaling of the FP is governed by its implicit derivative with respect to $l$ (Eqs.~\ref{eq:fp_RNN} and \ref{eq:fp_NODE}), which is proportional to the inverse Jacobian, e.g., $d\bm{z}^*/dl = -\bm{J}(l)^{-1} (\partial \bm{f}_\theta/\partial l)$ for the continuous-time model. The partial derivative $\partial \bm{f}_\theta/\partial l = \bm{W}_v \bm{\psi}(\bm{z}^*) + \bm{h}_v$ is strictly bounded because $\bm{\psi}$ is bounded. At the same time, because the model Jacobian scales through an affine mapping with $l$, its inverse scales as $\|\bm{J}(l)^{-1}\| = \mathcal{O}(|l|^{-1})$. Consequently, the rate of change of the fixed point location is bounded by $\left\| d\bm{z}^*/dl \right\| \leq \mathcal{O}(|l|^{-1})$. Integrating this differential inequality dictates that the fixed point can grow at most logarithmically: $\|\bm{z}^*(l)\| = \mathcal{O}(\ln |l|)$. 

Conversely, consider when $\bm{\psi}$ is an unbounded but polynomially bounded function of degree 1 (i.e., ``ReLU-like'' activations such as ReLU, GELU, or ELU). To determine its maximum possible growth, assume that as the latent feature diverges ($l \to \infty$), the extrapolated fixed point also grows large and eventually enters a stable linear regime. For strictly piecewise functions like ReLU, this means remaining within a single activation orthant; for smooth variants like GELU or ELU, this means reaching the asymptotic limit where the function behaves linearly (e.g., $\bm{\psi}(x) \approx x$ for $x \gg 1$). In either case, the activation effectively acts as a constant diagonal projection matrix $\bm{P}$, such that $\bm{\psi}(\bm{z}^*) \approx \bm{P}\bm{z}^*$ (with exact equality for ReLU). The equilibrium condition expands to 
$$(\bm{W}_c + l\bm{W}_v)\bm{P}\bm{z}^* + (\bm{h}_c + l\bm{h}_v) \approx \bm{0}.$$ 
Solving for the fixed point yields 
$$ \bm{z}^*(l) \approx -(\bm{W}_c \bm{P} + l\bm{W}_v \bm{P})^{-1}(\bm{h}_c + l\bm{h}_v). $$ 
As $l \to \infty$, the terms linear in $l$ dominate. The matrix inverse contributes a scaling factor of $1/l$, which algebraically cancels the $l$ scaling from the bias vector, resulting in a finite constant limit:
\begin{equation}
    \lim_{l \to \infty} \bm{z}^*(l) = -(\bm{W}_v \bm{P})^{-1} \bm{h}_v.
    \label{eq:fp_limit}
\end{equation}
Because assuming unbounded growth leads to a finite constant, we reach a contradiction. This structural cancellation bounds the asymptotic growth of the fixed point, ensuring $\|\bm{z}^*(l)\| = \mathcal{O}(1)$.

Thus, in both architectural cases, the affine parameter dependence structurally restricts the model's fixed point to scale either logarithmically or to remain constant as a function of the feature $l$. Because the latent feature is forced to scale as $|l(p)| = \Theta(p^n)$ to successfully capture the true system's dynamical scaling, the model's fixed point is strictly bounded by $\mathcal{O}(\ln p)$ or $\mathcal{O}(1)$ with respect to the control parameter $p$. However, the true physical system's fixed point exhibits independent algebraic scaling, $\|\bm{z}^*_{\text{true}}(p)\| = \Theta(p^m)$, in general with $m \neq 0$. Because these functional classes are asymptotically incompatible, the spatial discrepancy $\|\bm{z}^*_{\text{true}}(p) - \bm{z}^*_{\text{model}}(p)\|$ diverges strictly as $p \to \infty$. This mathematically induces an irreducible geometric truncation error, proving that the model must eventually fail to capture the true topological structure outside the training domain. $\square$

\subsection{Proof of Corollary~\ref{cor:cor_1_1}}
\label{appx:proof_cor1_1}

\textit{Proof.} Assuming the true physical system exhibits linear dynamical scaling in the control parameter $p$ (i.e., its Jacobian scales as $\Theta(p)$), the model's latent feature must scale as $l(p) = \Theta(p)$ to correctly match the local growth rates. As established in the proof of Theorem~\ref{theorem:theorem_1} (Appx.~\ref{appx:proof_thm1}), the affine parameter dependence imposes strict asymptotic bounds on the norm of the model's FP, depending on the activation function $\bm{\psi}$. Specifically, for globally bounded activations (e.g., $\tanh$), the norm of the FP is bounded by $\|\bm{z}^*(l)\|_2 = \mathcal{O}(\ln |l|)$. For polynomially bounded activations of degree 1 (e.g., ReLU), the linear scaling of the weights cancels out (Eq.~\ref{eq:fp_limit}), limiting the norm to a constant upper bound, $\|\bm{z}^*(l)\|_2 = \mathcal{O}(1)$. Because the latent feature must scale linearly with the control parameter ($l(p) = \Theta(p)$) to match the true system's local growth rates, the asymptotic behavior of the model's FP is strictly bounded by either $\|\bm{z}^*(p)\|_2 = \mathcal{O}(\ln |p|)$ or $\|\bm{z}^*(p)\|_2 = \mathcal{O}(1)$.

\textbf{Divergence directly at the boundary:} In either case, this architectural constraint prevents the model from independently matching the true physical system's positional scaling. Let $p_{\text{max}}$ be the upper boundary of the training domain. Because the path of the model's FP, $\bm{z}^*_{\text{model}}(p)$, belongs to a fundamentally different algebraic class than the true system's FP path, $\bm{z}^*_{\text{true}}(p)$, their Taylor series expanded at the boundary $p_{\text{max}}$ cannot globally align. Even under the optimal assumption that the model perfectly captures the true system's position and first derivative at the boundary (i.e., the linear dependence on the control parameter), there exists a lowest-order derivative $\kappa \geq 2$ where the expansions diverge. Consequently, for any finite extrapolation $p = p_{\text{max}} + \Delta p$, a spatial truncation error emerges immediately, scaling as $\mathcal{O}(\Delta p^\kappa)$. As $\Delta p$ increases, this spatial mismatch causes perturbations in the eigenvalue spectrum of the Jacobian at the FP. This non-physical spectral drift permits eigenvalues to cross the stability boundary, potentially inducing spurious bifurcations and topological failure. The same argument applies analogously to the lower boundary of the training domain. $\square$

\subsection{Proof of Corollary~\ref{cor:cor_1_2}}
\label{appx:proof_cor1_2}

\textit{Proof.} Assume the true continuous-time system is governed by a scalar control parameter $p$, with the true Jacobian's variational dependence given by $\bm{J}_{\text{true}}(p) = \bm{J}_0 + p \bm{J}_1$. As defined in Proposition~\ref{prop:prop_1} (Appx.~\ref{appx:proof_prop1}), the matrix $\bm{J}_1$ belongs to a sparse $k$-dimensional subspace $\mathcal{S} \subset \mathbb{R}^{N \times N}$. Physically, this restricts the true parameter's influence to a sparse set of exactly $k$ specific matrix entries, meaning the remaining $N^2 - k$ entries of $\bm{J}_1$ are zero.

Assume the DSR model employs an unregularized, multi-dimensional affine feature mapping $\bm{l}: \mathbb{R} \to \mathbb{R}^L$ with $L \geq 2$. The model's Jacobian evaluates to: 
$$ \bm{J}_{\text{model}}(\bm{l}(p)) = \hat{\bm{J}}_0 + \sum_{i=1}^L l_i(p) \hat{\bm{J}}_{1,i} $$ 
By Theorem~\ref{theorem:theorem_1}, to simultaneously satisfy decoupled positional ($\Theta(p^m)$) and dynamical ($\Theta(p^n)$) scaling on the training domain without algebraic entanglement, the mapping $\bm{l}(p)$ must utilize at least two asymptotically distinct functional forms. Therefore, as the control parameter is evaluated far outside the training domain ($|p| \to \infty$), the norm of the feature vector must grow without bound, $\|\bm{l}(p)\|_2 \to \infty$.

Following the decomposition introduced in the proof of Proposition~\ref{prop:prop_1} (Appx.~\ref{appx:proof_prop1}), we separate each learned feature-coupling matrix into physical and spurious orthogonal components using the Frobenius inner product: $\hat{\bm{J}}_{1,i} = \hat{\bm{J}}_{1,i}^{\parallel} + \hat{\bm{J}}_{1,i}^{\perp}$, where $\hat{\bm{J}}_{1,i}^{\parallel} \in \mathcal{S}$ and $\hat{\bm{J}}_{1,i}^{\perp} \in \mathcal{S}^{\perp}$. The true system's dependence on the control parameter is strictly one-dimensional, governed solely by $p\bm{J}_1$. Because the model maps this scalar parameter into an $L$-dimensional feature space, it is structurally over-parameterized. Without rank or sparsity priors, the underlying physics places no constraints on the $(N^2 - k) \times L$ degrees of freedom that map into the non-physical subspace $\mathcal{S}^{\perp}$. Since ERM strictly optimizes for in-domain training loss, it lacks any mechanism or incentive to drive these unconstrained weights to zero. Consequently, these excess dimensions can remain non-zero, yielding a dense orthogonal projection: 
$$ \bm{E}^{\perp}(p) = \sum_{i=1}^L l_i(p) \hat{\bm{J}}_{1,i}^{\perp} \neq \bm{0} $$

Consider the extrapolation error $\bm{E}(p) = \bm{J}_{\text{model}}(\bm{l}(p)) - \bm{J}_{\text{true}}(p)$. Because $\bm{E}^{\perp}(p)$ resides in the orthogonal complement to the true physical scaling, it cannot be canceled by the true dynamics. Thus, the residual error is bounded from below by this orthogonal leakage:
$$ \|\bm{E}(p)\|_2 \geq \|\bm{E}^{\perp}(p)\|_2. $$ Because the matrices $\hat{\bm{J}}_{1,i}^{\perp}$ are fixed and dense, and the feature vector diverges far outside the training domain ($\|\bm{l}(p\to\pm\infty)\|_2 \to \infty$), the magnitude of the orthogonal leakage $\|\bm{E}^{\perp}(p)\|_2$ inevitably diverges when evaluated out-of-domain. Specifically, this spurious leakage grows at a rate proportional to the fastest-growing component of the feature vector, scaling as $\Theta(|p|^{\max(m,n)})$.

To translate this growing matrix error into the system's dynamic behavior, we recall the spectral bound established via the Bauer--Fike theorem in the proof of Proposition~\ref{prop:prop_1} (Appx.~\ref{appx:proof_prop1}). For any eigenvalue $\tilde{\lambda}$ of the extrapolated model Jacobian, its deviation from a true eigenvalue $\lambda$ is bounded by the matrix error: 
$$ |\tilde{\lambda} - \lambda| \leq \kappa(\bm{V}) \|\bm{E}(p)\|_2. $$ Because the orthogonal leakage $\|\bm{E}(p)\|_2$ diverges as $\Theta(|p|^{\max(m,n)})$ when the control parameter scales out-of-domain, the potential perturbation of the model's eigenvalues becomes completely unbounded. Consequently, the model's eigenvalues are free to drift infinitely far away from the true physical eigenvalues. As this non-physical spectral drift grows, the eigenvalues can be pushed across the system's stability boundaries, inducing spurious bifurcations and thus leading to reconstruction failures. $\square$

\subsection{Proof of Theorem~\ref{theorem:theorem_2}}
\label{appx:proof_thm2}

\textit{Proof.} Let the true continuous-time system be governed by the continuous-time Jacobian $\bm{J}_{\text{cont}}(\bm{x}, p) = \bm{J}_0(\bm{x}) + p\bm{J}_1(\bm{x})$. The exact discrete-time Jacobian, mapping the local dynamics over a finite time interval $\Delta t$, must account for the fact that the continuous-time Jacobians evaluated at different points along the trajectory generally do not commute. Therefore, it is given by the time-ordered exponential \cite{blanes2009magnus}:
\begin{equation}
    \bm{J}_{\text{disc}}(\bm{x}, p) = \mathcal{T}\exp\!\left(\int_0^{\Delta t} \bm{J}_{\text{cont}}(\boldsymbol{\Phi}_s(\bm{x}), p) ds\right)
\end{equation}
where the operator $\mathcal{T}$ enforces chronological ordering, ensuring that matrices evaluated at later times appear to the left of those evaluated at earlier times in the expansion.

For notational clarity, in the subsequent derivation we define the shorthand $\bm{J}(s) \equiv \bm{J}_{\text{cont}}(\boldsymbol{\Phi}_s(\bm{x}), p)$ for the continuous-time Jacobian evaluated along the local trajectory. Furthermore, we hereafter suppress the explicit initial state dependence $\bm{x}$ on the left-hand side, writing $\bm{J}_{\text{disc}}(p)$. Expanding the integral via the Peano-Baker series~\cite{rugh1996linear} yields the exact matrix expansion:
\begin{equation}
    \bm{J}_{\text{disc}}(p) = \bm{I} + \int_0^{\Delta t} \bm{J}(s) ds + \int_0^{\Delta t} \bm{J}(s_1) \int_0^{s_1} \bm{J}(s_2) ds_2 ds_1 + \mathcal{O}(\Delta t^3)
\end{equation}

To evaluate these integrals, we Taylor-expand $\bm{J}(s)$ around $s=0$. Expanding around $s=0$ is mathematically required as it is the starting point of the discrete time step, where the system state $\bm{x}$ is exactly known ($\boldsymbol{\Phi}_0(\bm{x}) = \bm{x}$). This allows us to express the local flow over the short interval $\Delta t$ entirely in terms of the initial state and its time derivatives. Evaluating the initial matrix gives $\bm{J}(0) = \bm{J}_{\text{cont}}(\bm{x}, p) = \bm{J}_0 + p\bm{J}_1$. The expansion therefore is:
\begin{equation}
    \bm{J}(s) = \bm{J}(0) + \dot{\bm{J}}(0)s + \mathcal{O}(s^2)
\end{equation}

Substituting this into the Peano-Baker series and evaluating the integrals gives rise to the first- and second-order terms for the true discrete-time Jacobian:
\begin{align}
    \bm{J}_{\text{disc}}(p) &= \bm{I} + (\bm{J}_0 + p\bm{J}_1)\Delta t + \frac{1}{2}\left( (\bm{J}_0 + p\bm{J}_1)^2 + \dot{\bm{J}}_0 + p\dot{\bm{J}}_1 \right) \Delta t^2 + \mathcal{O}(\Delta t^3) \\
    &= \underbrace{\left[ \bm{I} + \bm{J}_0 \Delta t + \frac{1}{2}(\bm{J}_0^2 + \dot{\bm{J}}_0) \Delta t^2 \right]}_{=:\bm{M}_0} \nonumber \\
    &\quad+ p \underbrace{\left[ \bm{J}_1 \Delta t + \frac{1}{2}(\bm{J}_0\bm{J}_1 + \bm{J}_1\bm{J}_0 + \dot{\bm{J}}_1)\Delta t^2 \right]}_{=:\bm{M}_1} \nonumber \\
    &\quad+ p^2 \left[ \frac{1}{2} \bm{J}_1^2 \Delta t^2 \right] + \mathcal{O}(\Delta t^3)
\end{align}

Now, let the discrete-time DSR model employ an affine feature mapping parameterized by $l$ (or $l_{\text{dyn}}$ for models with feature splitting). While the true discrete-time Jacobian above exhibits explicit quadratic dependence on $p$, the model's Jacobian is structurally constrained to be strictly linear with respect to its feature $l$:
\begin{equation}
    \bm{J}_{\text{model}}(l) = \hat{\bm{J}}_{0} + l \hat{\bm{J}}_{1} \ .
\end{equation}
Assume first that the model learns a linear feature mapping $l = \alpha p$ for some scalar $\alpha \neq 0$ (this assumption is relaxed to non-linear mappings $l=g(p)$ below). To minimize the in-domain error, ERM drives the affine model coefficients to align with the true system's linear Taylor components:
\begin{align}
    \hat{\bm{J}}_{0} &\approx \bm{M}_0 \\
    \alpha \hat{\bm{J}}_{1} &\approx \bm{M}_1
\end{align}

We analyze the extrapolation error $\bm{E}(p) = \bm{J}_{\text{model}}(\alpha p) - \bm{J}_{\text{disc}}(p)$, which isolates the structural mismatch. Because the affine model lacks any quadratic degrees of freedom with respect to $l$ (and thus $p$), the unmodeled $\mathcal{O}(p^2)$ physical curvature cannot be accounted for by the network parameters. Evaluating this difference yields the leading-order residual:
\begin{equation}
    \bm{E}(p) = - \frac{1}{2} p^2 \bm{J}_1^2 \Delta t^2 + \mathcal{O}(p^3 \Delta t^3)
\end{equation}

Crucially, the structural limitation persists even if the model learns a non-linear, injective feature mapping $l = g(p)$ (Eq.~\ref{eq:g}). Assuming the mapping captures the true system's linear dependence on the control parameter, we require $g'(0) \neq 0$ (guaranteeing local invertibility via the Inverse Function Theorem; if $g'(0)=0$, the model fails trivially at order $\mathcal{O}(p)$, i.e., it does not capture the linear dependence). Taylor-expanding the mapping around $p=0$ yields $g(p) = g(0) + g'(0)p + \frac{1}{2}g''(0)p^2 + \mathcal{O}(p^3)$. Substituting the expansion into the model Jacobian gives:
\begin{align}
\bm{J}_{\text{model}}(g(p)) &= \left( \hat{\bm{J}}_{0} + g(0)\hat{\bm{J}}_{1} \right) + p \left( g'(0)\hat{\bm{J}}_{1} \right) \nonumber \\
&\quad+ p^2 \left( \frac{1}{2}g''(0)\hat{\bm{J}}_{1} \right) + \mathcal{O}(p^3) \ .
\end{align}
To minimize the in-domain error at linear order $\mathcal{O}(p)$, empirical risk minimization forces $g'(0)\hat{\bm{J}}_{1} = \bm{M}_1$. Because $g$ is invertible, $g'(0) \neq 0$, which algebraically restricts $\hat{\bm{J}}_{1}$ to be a scalar multiple of $\bm{M}_1$. Consequently, isolating the second-order residual $\bm{E}(p) = \bm{J}_{\text{model}}(g(p)) - \bm{J}_{\text{disc}}(p)$ yields:
\begin{equation}
\bm{E}(p) = \frac{1}{2} p^2 \left[ \frac{g''(0)}{g'(0)} \bm{M}_1 - \bm{J}_1^2 \Delta t^2 \right] + \mathcal{O}(p^3 \Delta t^3)
\end{equation}
The non-linear feature map $g(p)$ merely introduces a scalar coefficient ($g''(0)/g'(0)$) that multiplies with the existing first-order matrix $\bm{M}_1$. Because the model's Jacobian parameterization remains strictly affine, the independent, squared matrix term $\bm{J}_1^2$ required by the true discrete-time flow does not occur in this mapping. Thus, the $\mathcal{O}(p^2)$ structural mismatch remains irreducible, regardless of the feature map's non-linearity.

As the model is evaluated at control parameters outside the training domain ($p \not\in [p_\mathrm{min}, p_\mathrm{max}]$), the magnitude of the truncation error, $\|\bm{E}(p)\|_2$, diverges quadratically. The model's strictly affine matrix dependence on control parameters systematically departs from the true discrete-time system's quadratic parameter dependence, giving rise to an irreducible structural mismatch that leads to topological failure. $\square$

\end{document}